\documentclass{article}
\usepackage{microtype}
\usepackage{graphicx}
\usepackage{subfigure}
\usepackage{booktabs} 

\usepackage{hyperref}

\usepackage[accepted]{icml2024}

\usepackage{amsmath}
\usepackage{amssymb}
\usepackage{mathtools}
\usepackage{amsthm}

\usepackage[capitalize,noabbrev]{cleveref}

\theoremstyle{plain}
\newtheorem{theorem}{Theorem}[section]
\newtheorem{proposition}[theorem]{Proposition}
\newtheorem{lemma}[theorem]{Lemma}
\newtheorem{corollary}[theorem]{Corollary}
\theoremstyle{definition}

\theoremstyle{remark}
\newtheorem{remark}[theorem]{Remark}

\usepackage[textsize=tiny]{todonotes}

\icmltitlerunning{Online Resource Allocation with Non-Stationary Customers}

\begin{document}
\twocolumn[
\icmltitle{Online Resource Allocation with Non-Stationary Customers}


\begin{icmlauthorlist}
\icmlauthor{Xiaoyue Zhang}{yyy}
\icmlauthor{Hanzhang Qin}{yyy}
\icmlauthor{Mabel C. Chou}{yyy}

\end{icmlauthorlist}

\icmlaffiliation{yyy}{Institute of Operations Research and Analytics, National University of Singapore, Singapore 117602}


\icmlcorrespondingauthor{Xiaoyue Zhang}{xiaoyue.z@u.nus.edu}
\icmlcorrespondingauthor{Hanzhang Qin}{hzqin@nus.edu.sg}
\icmlcorrespondingauthor{Mabel C. Chou}{mabelchou@nus.edu.sg}

\vskip 0.3in
]



 \printAffiliationsAndNotice{}  

\begin{abstract}
We propose a novel algorithm for online resource allocation with non-stationary customer arrivals and unknown click-through rates. We assume multiple types of customers arrive in a nonstationary stochastic fashion, with unknown arrival rates in each period, and that customers' click-through rates are unknown and can only be learned online. By leveraging results from the stochastic contextual bandit with knapsack and online matching with adversarial arrivals, we develop an online scheme to allocate the resources to nonstationary customers. We prove that under mild conditions, our scheme achieves a ``best-of-both-world'' result: the scheme has a sublinear regret when the customer arrivals are near-stationary, and enjoys an optimal competitive ratio under general (non-stationary) customer arrival distributions. Finally, we conduct extensive numerical experiments to show our approach generates near-optimal revenues for all different customer scenarios. 
\end{abstract}

\section{Introduction}

The realm of online resource allocation, critical in fields ranging from online advertising to traffic management, poses a substantial challenge: how to effectively and dynamically distribute limited resources in response to ever-evolving consumer behaviors. This task is particularly arduous in environments where consumer preferences fluctuate rapidly, rendering traditional static allocation models ineffective.

Addressing this challenge, we introduce the Unified Learning-while-Earning (ULwE) Algorithm, a novel approach embedded within the Contextual Bandit with Knapsacks (CBwK) framework. The ULwE algorithm stands out for its real-time adaptability to changing consumer preferences and uncertain click-through rates, a notable departure from traditional methods that often fail to capture the non-stationary nature of customer arrivals.

This paper is structured as follows: \cref{sec: main contribution} summarizes our main contributions. \cref{sec: literature} reviews relevant literature, setting the stage for our innovation. \cref{sec:model} introduces our models for online resource allocation, adapted for non-stationary environments. \cref{section:algo} discusses the ULwE Algorithm in detail, presenting our unified theoretical framework. Finally, \cref{sec: experiment} validates our algorithm through empirical studies on simulated data and \cref{sec: conclusion} concludes the paper.

\subsection{Basic Setup and Contributions}
\label{sec: main contribution}
\subsubsection{Basic setup}
Our online resource allocation problem encompasses $n$ resources, each defined by a revenue parameter ($r_i$), capacity ($c_i$), and a latent variable ($\theta_i$) within a predefined space $\Theta$. It operates over $T$ time periods, with each period $t \in [T]$ witnessing the arrival of a customer, characterized by a feature vector $x^t \in \mathbb{R}^d$. The likelihood of a customer purchasing from resource $i$ is modeled by $f_i(x^t, \theta_i^*)$, where $\theta_i^*$ is the true value of $\theta_i$, and $f_i(x^t, \theta_i) \leq 1$ ensures model validity. We assume the $\theta_i^*$s are unknown a priori, and we use $|\Theta|$ to denote the cardinality of the parameter space $\Theta$. 

The context vector $x^t$ for each time period $t$ is independently drawn from a distribution $P(x^t = x^{(l)}) = \mu_l^t$ for each customer type $l \in {1, 2, \ldots, L}$. Given the variability of this distribution across different time periods, the customer arrival process is inherently non-stationary. For every customer type represented by $x^{(l)}$, the expected number of arrivals, denoted as $\lambda_l$, is computed over the time horizon as $\lambda_l = \sum_{t=1}^T \mu_l^t$. We assume all $\mu_l^t$s are unknown but the $\lambda_l$s are given in the beginning. In practice, $\lambda_l$s correspond to the total arrival rate of the contexts over a long horizon and can typically be estimated with high accuracy. This information can be regarded as advice, and it is shown by \cite{Lyu2023NonStationaryBW} without any advice the nonstationary BwK (CBwK with identical contexts) admits a worst-case regret linear in $T$.

\subsubsection{Main Contributions}
Our primary contribution lies in the development of the ULwE Algorithm, a novel approach within the CBwK framework under non-stationary arrivals. This algorithm uniquely offers simultaneous guarantees of sub-linear regret and a constant competitive ratio (CR), effectively addressing the complexities of unknown click-through rates and non-stationary customer arrivals in dynamic environments.

\textbf{A Unified Learning-while-Earning Algorithm Framework with Sublinear Regret and Constant CR.} We propose a unified study of two essential performance guarantees in online resource allocation: regret and CR. Achieving sublinear regret is crucial in online advertising and resource allocation systems, as it signifies that the regret incurred by the algorithm grows at a slower rate compared to the time horizon. Additionally, the constant CR ensures that our algorithm achieves a consistently high-performance level in comparison to an optimal offline policy. While each guarantee has been explored separately, our main contribution lies in investigating them together within a unified algorithm framework and achieving the best of both worlds. 
\begin{figure}[ht]
\vskip 0.2in
\begin{center}
\centerline{\includegraphics[width=0.8\columnwidth]{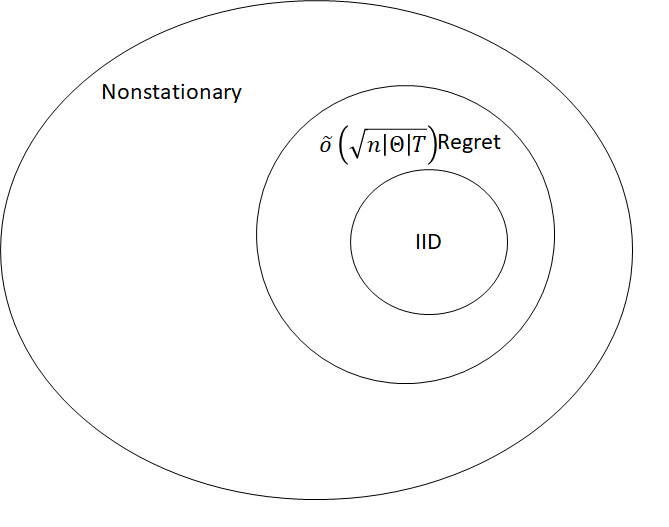}}
\caption{Problem categories}
\label{fig:regime}
\end{center}
\vskip -0.2in
\end{figure}

\cref{fig:regime} illustrates a graphical representation of three ellipsoids to visually depict this concept. The smallest ellipsoid represents the problem category of IID arrivals. The intermediate ellipsoid represents problems under near-i.i.d. arrivals scenario, in which our algorithm achieves an expected regret bound of $\tilde{O}(\sqrt{n|\Theta|T})$. The largest ellipsoid encompasses all types of nonstationary customer arrivals, including the most challenging scenarios. In this case, our algorithm provides a unified regret bound, which offers sublinear regret with a constant CR performance, allowing for effective resource allocation even in highly dynamic environments. Overall, under nonstationary arrivals, our ULwE algorithm recovers $\tilde{O}(\sqrt{n|\Theta|T})$ expected regret under near-i.i.d. arrivals (the exact condition is specified in \cref{section:algo}) and guarantees in general
\begin{align*}
\text{OPT} &\leq \left(1+ \frac{(1+\min\limits_{i\in [n]} c_i) \left(1-e^{-1/\min\limits_{i\in [n]}c_i }\right)}{1-1/e}\right)\mathbb{E}[\text{ALG}] \\
&\quad  +\tilde{O}(\sqrt{n|\Theta|T}),
\end{align*}

where $\text{OPT}$ denotes the expected revenue the optimal online algorithm that knows the true latent variable achieves, $\text{ALG}$ denotes the revenue our proposed online learning algorithm achieves (the latent variable is unknown in priori). 
When $\min\limits_{i\in [n]} c_i$ is sufficiently large, it can be shown the above guarantee can be rewritten as $\text{OPT}\leq (1+1/(1-1/e))\mathbb{E}
[\text{ALG}]+\tilde{O}(\sqrt{n|\Theta|T})$, or equivalently, $\mathbb{E}[\text{ALG}]\geq (1-1/e)\text{OPT} - \tilde{O}(\sqrt{n|\Theta|T})$. This guarantee is the tightest possible given the underlying contextual arrival process is adversarial and the click-through rates are unknown (see Theorem 4, \citealt{wangchi2022}). 

 \textbf{Non-Stationary Customer Arrivals.} 
Our approach to non-stationary customer arrivals encompasses two distinct scenarios: those with stochastic distributions and adversarial ones. For stochastic distributions, we employ past data to create empirical distributions, aiding in future predictions. This is achieved using the Upper Confidence Bound (UCB) algorithm and the Deterministic Linear Programming (DLP) formulation, optimizing resource allocation by maximizing expected revenue within the confidence bounds of demand rates.

In contrast, for adversarial or unknown distribution scenarios, we adopt a greedy algorithm, as suggested by \cite{wangchi2022}, which factors in discounted revenue and inventory constraints. This method ensures efficient allocation by considering both revenue maximization and resource limitations.

Our Unified Learning-while-Earning (ULwE) algorithm amalgamates these strategies, seamlessly transitioning between them based on the nature of customer arrivals. This integration, alongside inbuilt condition checks, endows our algorithm with the versatility to operate effectively across varying arrival patterns without prior knowledge of their sequence type. Such adaptability makes it well-suited for dynamic real-world contexts, addressing the challenges posed by diverse customer arrival processes.

\textbf{Unknown Click-Through Rates.} 
Our algorithm addresses the complexities introduced by unknown click-through rates in online advertising and resource allocation systems. It dynamically estimates these rates by leveraging historical data alongside current information, ensuring adaptive and informed decision-making.

Central to our approach is the use of a parametric model with a pre-defined prior distribution for the parameter $\theta \in \Theta$, allowing us to capture the nuanced relationships between different customer types and their purchasing probabilities. Initially, our algorithm samples a subset of $\theta$ values from existing data to construct the set $\Theta$, which is then continuously refined as new data becomes available. This process enables our algorithm to evolve with and respond to the changing click-through rate landscape, ensuring resource allocation decisions are both current and data-informed.

This strategy of integrating historical insights with real-time updates provides a well-rounded perspective on customer behavior, crucial for optimizing resource allocation in the dynamic field of online advertising.
\subsection{Literature Review}
\label{sec: literature}
\textbf{Contextual Bandit with Knapsack under Non-Stationary Arrivals.}
CBwK problems in non-stationary environments, integral to our research, have been a focal point in recent studies. Chen et~al. \yrcite{chen2013combinatorial} were pioneers in contextualizing CBwK within online advertising, proposing a novel greedy algorithm focusing on the reward-to-cost ratio. This groundbreaking approach opened avenues for diverse methodologies. Badanidiyuru et~al. \yrcite{badanidiyuru2014resourceful} introduced a probabilistic method, selecting optimal strategies from policy sets. Agrawal et~al. \yrcite{Agrawal2016} followed with a computationally efficient variant. Complementing these, Agrawal and Devanur \yrcite{AgrawalDevanur2016} developed LinCBwK, which utilizes optimistic estimates from confidence ellipsoids to adjust rewards for each action. Recently, Slivkins and Foster \yrcite{Slivkins2022EfficientCB} integrated LagrangeBwK \cite{Immorlica2019} with SquareCB, merging computational efficiency with statistical optimality.

In tackling non-stationary BwK problems without contextual information, Immorlica et~al. \yrcite{Immorlica2019} achieved an $O(\log T)$ CR against fixed distribution benchmarks in adversarial settings. This method's potential for sub-linear-in-T regret under specific conditions was further explored by Sivakumar et~al. \yrcite{Sivakumar2022SmoothedAL}, Liu et~al. \yrcite{Liu2022NonstationaryBW}, and Lyu et~al. \yrcite{Lyu2023NonStationaryBW}, offering critical insights into adaptable algorithms in dynamic environments.

\textbf{Online Resource Allocation.}
Our work intersects significantly with online resource allocation (also known as the AdWord problem, see \citealt{mehta2007adwords}), particularly in the assignment of resources to real-time, random demands. The CR serves as a key performance metric in this domain. Karp et~al.'s investigation \yrcite{Karp} into one-sided bipartite arrivals revealed the optimality of a simple $1/2-$competitive greedy algorithm among deterministic approaches. Mehta et~al.'s study \yrcite{1530720} on the Adwords problem, modeled as a linear program (LP), led to an online algorithm with a $1- 1/e$ CR based on primal-dual LP analysis. Fahrbach et~al. \yrcite{fahrbach2022edge} further broke new ground with the online correlated selection subroutine, surpassing the $1/2$ barrier and reaching a CR of at least $0.5086$. 

The works most aligned with ours include those by Ferreira et al. \yrcite{ferreira2018online}, Zhalechian et~al. \yrcite{zhalechian2022online}, and Cheung et~al. \yrcite{wangchi2022}. Ferreira et al.'s \yrcite{ferreira2018online} dynamic pricing algorithms, grounded in Thompson sampling, highlighted versatility in handling resource-constrained multi-armed bandit problems. Zhalechian et al. \yrcite{zhalechian2022online} addressed personalized learning resource allocation with a novel general Bayesian regret analysis, adept for adversarial customer contexts. Cheung et~al. \yrcite{wangchi2022} focused on the allocation of limited resources over time to heterogeneous customers, a methodology that has significantly influenced our approach to CR analysis.

\section{Model}
\label{sec:model}
Our objective is to maximize the total reward obtained from customers' clicks or purchases while considering budget constraints for the available resources or advertisers. We consider a system with $n$ resources, resource $i\in[n]$ has an initial budget value $c_i$. Customers arrive at the system from period $1$ to $T$. When a customer arrives, we observe the feature vector $x \in \mathbb{R}^d$ of the customer, where $d$ is the length of the feature vector. At this point, the system must make a decision: either reject the customer irrevocably or assign a resource to the customer immediately. If the resource $j$ is assigned, the customer's likelihood of purchasing the resource is determined by the probability function $f_j(x)$ associated with that resource $j$ and the customer's feature vector $x$. If the customer makes a purchase, the system earns a reward $r_j$ that is deducted from the budget of the assigned resource. If the customer does not make a purchase, no reward is earned.

To address this complex problem, our model incorporates two key elements: non-stationary customer arrivals and unknown click-through rates. We recognize that customer behavior may change over time, requiring the system to adapt to these fluctuations. Additionally, the exact click-through rates, which represent the likelihood of customers clicking on a resource or making a purchase, are initially unknown. However, we utilize extensive data to estimate and update these rates in real-time.

\subsection{Handling Non-Stationary Customer Arrivals}

In real-world systems, customer arrivals often exhibit non-stationary behavior, reflecting temporal variations in customer types. For example, students accessing a system are likely to show different patterns during mornings, evenings, and class hours. We categorize non-stationary customer arrivals into two types for our model:

\begin{itemize}
\item \textbf{Stochastic Arrivals with Non-Stationary Distributions:} These are regular scenarios without disruptive events. Past data can be used to learn distributions for future arrivals.
\item \textbf{Adversarial Arrivals with Agnostic Distributions:} Such scenarios occur during unpredictable events (e.g., Black Friday), where arrivals are assumed to be controlled by an adversary.
\end{itemize}

Our model initially focuses on strategies for stochastic arrivals, then extends to address adversarial scenarios, ensuring adaptability to various real-world contexts. We assume no prior knowledge of arrival patterns except for the total arrival rate $\lambda_j$ for each customer type $j \in [L]$. This approach, suitable for both IID and adversarial sequences, addresses the lack of adaptability of previous approaches since it is typically impossible to tell in advance whether future arrivals will be near-stationary or not.

\subsection{Modeling Unknown Click-Through Rates}
Addressing unknown click-through rates (CTR) is pivotal, particularly in contexts with ample historical data. While historical data provides insights into individual preferences, actual CTRs are influenced by recent factors like ad quality and competitor promotions. Our model adopts a parametric approach, beginning with a known prior distribution for $\theta \in \Theta$. This setup allows for a flexible, parametric formulation of the purchase probability function $f(x^t, \theta)$, where $x^t \in \mathbb{R}^d$ is the feature vector for customer arrivals at time $t \in [T]$. 

An example model choice is the Polynomial logistic regression model \citep{richardson2007predicting}, with $\theta$ representing polynomial term parameters. To utilize historical data, our algorithm samples a subset of $\theta$ values initially, forming the set $\Theta$. This set is dynamically updated over time, enabling the algorithm to adapt to changing conditions and refine decision-making processes continually. In practice, $\Theta$ can include any type of the parameters involved in a parameter set associated with a particular machine learning model (e.g., a Large Language Model) for CTR prediction.

\subsection{An Upper Bound on the Optimal Revenue}
In this section, our primary objective is to establish a theoretical upper bound on the optimal revenue in our model, denoted as $\overline{\text{OPT}}$. This upper bound provides a crucial benchmark for assessing the performance of various algorithms and understanding their efficiency and effectiveness in different resource allocation scenarios.

To establish an upper bound, we introduce a deterministic linear programming model $J^D$, which aims to maximize the deterministic revenue given the capacities and customer arrival patterns:
\begin{equation}
\begin{aligned}
J^D(c,t) = \max&_{s_{ij}, i\in [n], j\in[L]}  \sum_{i=1}^n \sum_{j=1}^L \sum_{s=1}^t r_i s_{ij}^s f_i(x^j,\theta^*) \\
 \text{s.t. } &\sum_{j=1}^L \sum_{s=1}^t {\mu_j}^s s_{ij}^s f_i(x^j,\theta^*) \leq c_i, \forall i \in [n] \\
& \sum_{i=1}^n\sum_{s=1}^t s_{ij}^s = \sum_{s=1}^t {\mu_j}^s, \forall j \in [L] \\
& s_{ij}^s \geq 0 ,\forall i \in [n], \forall j \in [L] \label{eq:benchmark}
\end{aligned}
\end{equation}

We establish that:
\begin{equation*}
\text{OPT} \leq J^{D}(c, t), 
\end{equation*}
indicating that $J^{D}(c, t)$ serves as an upper bound for the optimal revenue. Consequently, we define regret as the difference between this upper bound and the actual optimal revenue achieved, formulated as:
$\text{Regret} = \overline{\text{OPT}} - \text{ALG}.$

This quantifies the efficiency of different algorithms in approaching the theoretical maximum revenue. For a comprehensive understanding of the problem formulation, constraints, and detailed derivations, please refer to the \cref{appendix: benchmark}, which includes the full mathematical treatment and theoretical analysis supporting these findings.

\section{Algorithm and Analysis}
\label{section:algo}
We propose the ULwE algorithm, displayed in \cref{alg:algoULwE}, which is applicable in both stochastic and adversarial environments. The algorithm design involve constructing switch strategy to address the model uncertainty on $\mu$, as discussed in \cref{sec:alULwE}. In \cref{sec:main results}, we provide a regret upper bound to ULwE, and demonstrate the scheme has a sublinear regret when the customer arrivals are near-stationary, and enjoys an optimal competitive ratio under general (non-stationary) customer arrival distributions. In \cref{sec:proofsketch} we provide a sketch proof of the regret upper bound, where the complete proof is in \cref{Asec: regret proof}

\begin{algorithm}[tb]
   \caption{Unified Learning-while-Earning (ULwE)}
   \label{alg:algoULwE}
\begin{algorithmic}
   \STATE {\bfseries Input:} Resource capacities $c$, time horizon $T$, customer types $L$
   \STATE Initialize $\Omega_i^0 = \Theta_i$ for all $i \in [n]$
   \FOR{$t=1$ {\bfseries to} $T$}
   \IF{switch = FALSE}
   \STATE $I^t=\text{ALG}_{\text{LP}}(c,T,L,\Omega_i^{t-1})$.
   \STATE Check if conditions \eqref{eq:condition1} and \eqref{eq:condition2} for switching are met.
   \IF{any condition is violated}
   \STATE switch = TRUE.
   \ENDIF
   \ELSE
   \STATE  $I^t=\text{ALG}_{\text{ADV}}(c,T,L,\Omega_i^{t-1})$.
   \ENDIF
   \STATE Check if conditions \eqref{eq:update1} and \eqref{eq:update2} for updating $\Omega_i$ are met.
    \IF{any condition is violated}
   \STATE Remove $\bar \theta^t$ from $\Omega_{I^t}^t$
   \ELSE 
   \STATE Set $\Omega_i^t = \Omega_i^{t-1}$ for all $i\in [n]$
   \ENDIF
   \ENDFOR
\end{algorithmic}
\end{algorithm}

\subsection{The ULwE Algorithm}
\label{sec:alULwE}
In this section, we explore the ULwE algorithm (\cref{alg:algoULwE}), which is applicable in both stochastic and adversarial environments. ULwE is an evolution of the UCB paradigm, integrating aspects of inventory balancing with a penalty function \cite{doi:10.1287/mnsc.2014.1939}. A distinctive feature of ULwE is its dynamic approach: the algorithm continually updates the parameter set $\theta$ and monitors capacity constraints to determine the necessity of a policy switch. Specifically, ULwE adapts its policy to best suit customer patterns, whether they follow IID or adversarial arrival sequences. This strategy of switching sets ULwE apart from the traditional UCB-based methods used in the stochastic context \cite{AgrawalDevanur2016} and the strategies applied to the adversarial Bandit with Knapsack (BwK) settings \cite{Immorlica2019}I. While ULwE shares similarities with previous works on stochastic and adversarial bandits with knapsacks in the methodologies of arm selection, which involves LP optimization and inventory balancing, and in its fundamental updating mechanism, which is based on the UCB approach.

The algorithm operates through three crucial stages: selecting an arm to obtain rewards and consumption vectors, updating the confidence set $\Theta$, and assessing whether to switch states.

\textbf{Arm Selection.} In each period $t$, our algorithm offers a resource $I^t$  from a set of $n$ resources to the customer. This process constitutes the arm selection phase, which is executed differently under two distinct protocols: $\text{ALG}_{\text{LP}}$ and  $\text{ALG}_{\text{ADV}}$, catering to different customer arrival patterns. 

For $\text{ALG}_{\text{LP}}$ protocol (\cref{alg:alglp}) designed for environments with stationary customer arrivals, we assume that the probability of each customer type $l$ arriving at time $t$ remains constant i.e. $\mu_l^t=\mu_l$ for all $t\in [T]$. The algorithm solves a LP problem at each step to maximize revenue, constrained by inventory limits and the requirement that the sum of probabilities of offering all resources to each customer equals one. The LP formulation is as follows:
$ U^t =  $:
\begin{align}
\label{eq:LP2}
\begin{split}
\max_{s_{ij}, i\in [n], j\in[L]} & \sum_{i \in [n]} r_i \sum_{j \in [L]}\lambda_j s_{ij} \bar f_i (x^{(j)}, \Omega_i^{t-1})\\
\text{s.t. } & \sum_{j \in [L]} \lambda_j s_{ij} \bar f_i (x^{(j)}, \Omega_i^{t-1}) \leq c_i, \ \forall i \in [n]\\
& \sum_{i\in [n]}s_{ij}=1,~\forall j\in [L]\\
&  s_{ij} \geq 0, \ \forall i \in [n], j \in [L].
\end{split}
\end{align}
Here, the function $\bar f_i(x, \Omega)$ calculates the maximum purchase probability for resource $i$ given a set $\Omega$ of valid latent variables, i.e.$\bar f_i(x, \Omega) = \max_{w \in \Omega} \{ f_i(x, w)\}$.

In contrast, $\text{ALG}_\text{ADV}$ protocol (\cref{alg:algadv}) is designed to handle the dynamic and unpredictable nature of customer behavior and preferences. In this situation, the underlying distribution of $x_t$ no longer exists. Instead, the value of $x_t$ is chosen by an adversary in each period $t$. \cref{alg:algadv} computes a real-valued function,  $\Psi(x) := \frac{e^x - 1}{e - 1}$, defined over the interval $[0,1]$. This function is instrumental in adjusting the revenue values $r^t_i$ for each resource $i$, taking into account the resource's past utilization and its capacity. The adjusted revenue, denoted by $r_i^t$, is calculated as $r_i \times \left( 1 - \Psi \left( \frac{N_i^{t-1}}{c_i} \right) \right)$, where $N_i^{t-1}$ represents the previous consumption of resource $i$. The algorithm then selects the resource $I^t$ that maximizes the upper confidence bound of the single period revenue, taking into account the modified revenue values and the confidence intervals defined by the set of valid latent variables $\Omega_i^{t-1}$.

\begin{algorithm}[tb]
   \caption{ $\text{ALG}_{\text{LP}}$ Protocol}
   \label{alg:alglp}
\begin{algorithmic}
   \STATE {\bfseries Input:} Resource capacities $c$, time horizon $T$, customer types $L$
   \STATE Initialize $\Omega_i^0 = \Theta_i$ for all $i \in [n]$
   \FOR{$t=1$ {\bfseries to} $T$}
   \STATE Solve LP in \cref{eq:LP2} to obtain $\bar s^t, \bar \gamma^t$
    \STATE Select resource $i$ with probability $\bar s_{iJ^t}$.
   \ENDFOR
\end{algorithmic}
\end{algorithm}

\begin{algorithm}[tb]
   \caption{$\text{ALG}_{\text{ADV}}$ Protocol}
   \label{alg:algadv}
\begin{algorithmic}
   \STATE {\bfseries Input:} Resource capacities $c$, time horizon $T$, customer types $L$
   \STATE Initialize $\Omega_i^0 = \Theta_i$ for all $i \in [n]$
   \FOR{$t=1$ {\bfseries to} $T$}
   \STATE Observe the context $x^t$ of the new arrival in period $t$.
   \STATE Select $I^t = \arg\max\limits_{i\in [n]} r_i^t \bar{f}_i(x^t, \Omega_i^{t-1})$.
   \ENDFOR
\end{algorithmic}
\end{algorithm}

\textbf{Update Confidence Set.} In the concept of the UCB strategy, a key aspect is the continuous updating of its confidence set. In our algorithm, this process begins with the formation of a theta set $\Theta$, encompassing a set of possible parameter values. From this set, the confidence set $\Omega_i^t$ is constructed. After selecting the resource, the algorithm identifies the maximizer $\bar \theta$ of the purchase probability for that resource. The updating mechanism involves two critical equations, for all $ \theta \in \Omega^{t-1}$:
\begin{align}
&\left| \sum_{t' \in D_{I^{t}}^t(\bar \theta^t)}(f_{I^{t'}}(x^{t'},\bar \theta^t) - a^{t'})\right| \leq \sqrt{t\log\frac{2t}{\beta(n,T)}} \label{eq:update1} \\
&\left| \sum_{t'\in D_{I^t}^t(\bar \theta^t)} (f_{I^{t'}}(x^{t'},\bar \theta^t)- f_{I^{t'}}(x^{t'},\theta) ) \right|\leq \sqrt{t\log\frac{2t}{\beta(n,T)}} \label{eq:update2}
\end{align}
Here, $\beta(n,T)$ is set as $1/(nT)$. The set $D^t_i$ tracks the periods up to time $t$ where resource $i$ is offered. The customer's purchase decision is represented by $a^t \in {0,1}$.

\cref{eq:update1} ensures that the accumulated regret, the difference between predicted and actual purchases, remains within a predefined threshold. If this threshold is exceeded, $\bar \theta$ is removed from the confidence set, refining the purchase probability estimates. \cref{eq:update2} is crucial in maintaining the stability of the algorithm under near-IID arrival patterns. It limits the occurrence of a switch, fostering a more consistent and effective resource allocation strategy and leading to sublinear regret in these situations.

At its core, this iterative refinement process, by discarding less probable theta values, sharpens the focus of the confidence interval on the most promising parameter values. This refined approach, underpinned by constantly updated data, significantly enhances the algorithm’s ability to make efficient and accurate resource allocation decisions.

\textbf{Switch State Checking.} The ULwE algorithm checks if it has switched or not in each time step. When not switched, the algorithm checks two conditions. The switching conditions involve two key inequalities which are checked at each time step $t$. For all $\theta\in \Omega^{t-1}$:
\begin{align}
 & \left| \sum_{l=1}^t \sum_{i=1}^n r_i (s_{iJ^l}(\theta)-\bar s_{iJ^t}^l) f_{iJ^l}(x^{(J^t)},\theta) \right| \nonumber \\
 &\quad \leq  \max\limits_{i\in [n]}r_i \sqrt{32t\log(4|\Theta|t/\beta(n,T)) }, \label{eq:condition1}
 \end{align}
For all $i\in [n]$
\begin{align}
\sum_{l=1}^t \bar s_{iJ^l}^l\bar f_i(x^{(J^l)},\Omega^{l-1})  \leq    \frac{t}{T}c_i +\sqrt{2t\log(2t/\beta(n,T))} \label{eq:condition2}
\end{align} \cref{eq:condition1} ensures that the estimated upper bound regret does not grow beyond a sublinear rate. This condition measures the algorithm's performance in terms of regret, which quantifies the deviation between the total expected reward obtained by the algorithm and the maximum achievable reward that an optimal offline algorithm could obtain. 
\cref{eq:condition2} checks whether the remaining capacity of the resources is sufficient. This condition ensures that the cumulative resource allocation does not exceed the available capacity plus a certain tolerance level. It takes into account the capacity constraints imposed by the resources and ensures that the algorithm does not allocate more resources than what is available. If both conditions hold, the algorithm continues to the next time step. However, if any of the conditions is violated, indicating either excessive regret growth or insufficient resource capacity, the algorithm switches from the current algorithm ($\text{ALG}_{\text{LP}}$) to $\text{ALG}_{\text{ADV}}$. 

\subsection{Performance Guarantees of ULwE}
\label{sec:main results}
The following theorem provides a regret upper bound for \cref{alg:algoULwE}:

\begin{theorem}
\label{prop:hybrid regret}
In any case of nonstationary arrivals, the algorithm guarantees 
\begin{align*}
\text{OPT}\leq &\left(1+ \frac{(1+\min\limits_{i\in [n]} c_i) \left(1-e^{-1/\min\limits_{i\in [n]}c_i }\right) }{1-1/e}\right)  \mathbb{E}[\text{ALG}]\\
&+\tilde{O}(\sqrt{n|\Theta|T} ).
\end{align*}
When the arrivals are stationary, the algorithm guarantees
\begin{align*}
\text{OPT}\leq & \mathbb{E}[\text{ALG}]+\tilde{O}(\sqrt{n|\Theta|T} ).
\end{align*}
\end{theorem}
The complete proof of \cref{prop:hybrid regret} is proved in the Appendix, and we provide a sketch proof in \cref{sec:proofsketch}. The Therorem establishes an expected regret bound of $\tilde{O}(\sqrt{n|\Theta|T})$ for near-i.i.d. arrival scenarios and provides a unified regret bound for nonstationary arrivals, combining sublinear regret with a constant CR. In the context of resource allocation, these results are significant.The sublinear regret of our algorithm indicates that the gap between its performance and the optimal strategy narrows over time. This demonstrates the algorithm's ability to adapt and improve, becoming more effective with longer usage. The constant CR highlights that our algorithm's regret is always within a fixed factor of the optimal strategy's regret, regardless of the problem's scale. This underlines the algorithm's consistent effectiveness and robustness, even with changing customer preferences. A remark on the behavior of the CR under the scenario when $\min_{i\in [n]}c_i$ approaches infinity, are provided in the \cref{remark:regret}.

\subsection{Proof Sketch of \cref{prop:hybrid regret}}
\label{sec:proofsketch}
We provide an overview on the proof of \cref{prop:hybrid regret}, which is fully proved in \cref{Asec: regret proof}. First, we establish that $\text{ALG}_{\text{LP}}$ always incurs a sublinear regret and a linear rate of resource consumption with high probability as long as no switch occurs (Proposition \ref{prop:not switch sublinear regret}). Specifically, if $\text{ALG}_{\text{LP}}$ runs uninterrupted for $t$ iterations, the expected regret up to time $t$ is bounded.

\begin{theorem}
\label{prop:not switch sublinear regret}
Suppose $\text{ALG}_{\text{LP}}$ runs for $t$ iterations without being switched, then the expected regret up to time $t$ is upper bounded by
\begin{align*}
& 16\sqrt{2|\Theta|nt\log(4|\Theta|t/\beta(n,T))}+\sqrt{5}n\log(2t/\beta(n,T))\\
& +(2/t+2+\log t)\beta(n,T)(LP(\theta^*)+nt).
\end{align*}
Moreover, each resource $i\in [n]$ has at least
\begin{equation*}
\left[\frac{T-t}{T}c_i - \sqrt{8nt\log(2t/\beta(n,T))}\right]^+
\end{equation*}
remaining with probability at least $1-\beta(n,T)$.
\end{theorem}

We then demonstrate that under IID arrivals, the switch from $\text{ALG}_{\text{LP}}$ to $\text{ALG}_{\text{ADV}}$ does not occur with high probability (see \cref{prop:iid first condition} and \cref{prop:iid second condition}). Consequently, we can apply \cref{prop:not switch sublinear regret}, which assures sublinear regret without switch. This leads us to obtain $\tilde{O}(|\Theta|nT)$ regret under IID arrivals.

We further analyze our algorithm to address the case where the switch occurs, transitioning from $\text{ALG}_{\text{LP}}$ to $\text{ALG}_{\text{ADV}}$ at time $t$. \cref{thm:adversarial regret} quantify the performance of the optimal algorithm from time $t + 1$ to $T$ , given the consumption of all resources is zero.
\begin{theorem}
\label{thm:adversarial regret}
It holds that
\begin{align*}
\text{OPT} \leq & \frac{(1+\min\limits_{i\in [n]} c_i) \left(1-e^{-1/\min\limits_{i\in [n]}c_i }\right) }{1-1/e} \mathbb{E}[\textrm{ALG}]\\
& +\max\limits_{i\in [n]}r_i(\sqrt{|\Theta|n}+1)\sqrt{2T \log(2T/\beta(n,T)) }\\
&+ \max\limits_{i\in [n]} r_i (1/T+\log T+1)/n.
\end{align*}
\end{theorem}
This theorem is crucial as it sets an upper bound on the regret incurred by $\text{ALG}_{\text{ADV}}$ under adversarial arrival conditions. By examining the collective performance of both $\text{ALG}_{\text{LP}}$ and $\text{ALG}_{\text{ADV}}$, we derive a comprehensive regret bound in \cref{prop:hybrid regret}.  Our analysis extends further to encompass nonstationary arrivals, showing that our algorithm not only achieves sublinear regret but also maintains a constant CR.

\section{Numerical Studies}
\label{sec: experiment}
In this section, we conduct numerical experiments to assess the performance of $\text{ALG}_{\text{LP}}$, $\text{ALG}_{\text{ADV}}$, and the ULwE Algorithm in a controlled setting where resources are continuously sold. The main goal is to evaluate these algorithms under dynamic conditions with limited resource availability, using $J^D(c,t)$ (\cref{eq:benchmark}) as our benchmark (\cref{eq:benchmark}). Based on insights from \cite{doi:10.1287/mnsc.2014.1939}, our experimental setup includes periodic LP solutions every 50 periods to adapt to changing customer arrivals, factoring in inventory levels and empirical data. The experiments involve two customer types ($L=2$) and two resources ($n=2$), with customer purchase probabilities modeled through a logistic function. The ULwE Algorithm starts with historical data (500 observations) to estimate initial $\Theta$ values, comparing against an optimal offline algorithm’s maximum profit derived from known $\theta^{*}$ values. In this setup, Customer Type A exhibits a high purchase probability of 0.9, while Customer Type B maintains a consistent probability of 0.5, regardless of the resource. The total capacity of both resources matches the time horizon $T$, with revenues set at 1 and 1.5 units for Resource 1 and 2, respectively. Our analysis primarily assesses the ULwE Algorithm's performance, particularly its regret in both IID and adversarial customer arrival scenarios, highlighting its adaptability across various operational contexts. Experiment's regret values represent the average outcomes of 100 independent experiments, ensuring robustness and reliability of the results.

\subsection{Results under Near-IID Arrivals}
In this subsection, we analyze experimental results under IID arrival conditions, where customer arrivals are stable and predictable. We simulate Customer Types A and B arriving at rates of $0.6T$ and $0.4T$, respectively, representing a near-IID environment with consistent arrival probabilities.

\begin{figure}[ht]

\begin{center}
\centerline{\includegraphics[width=\columnwidth]{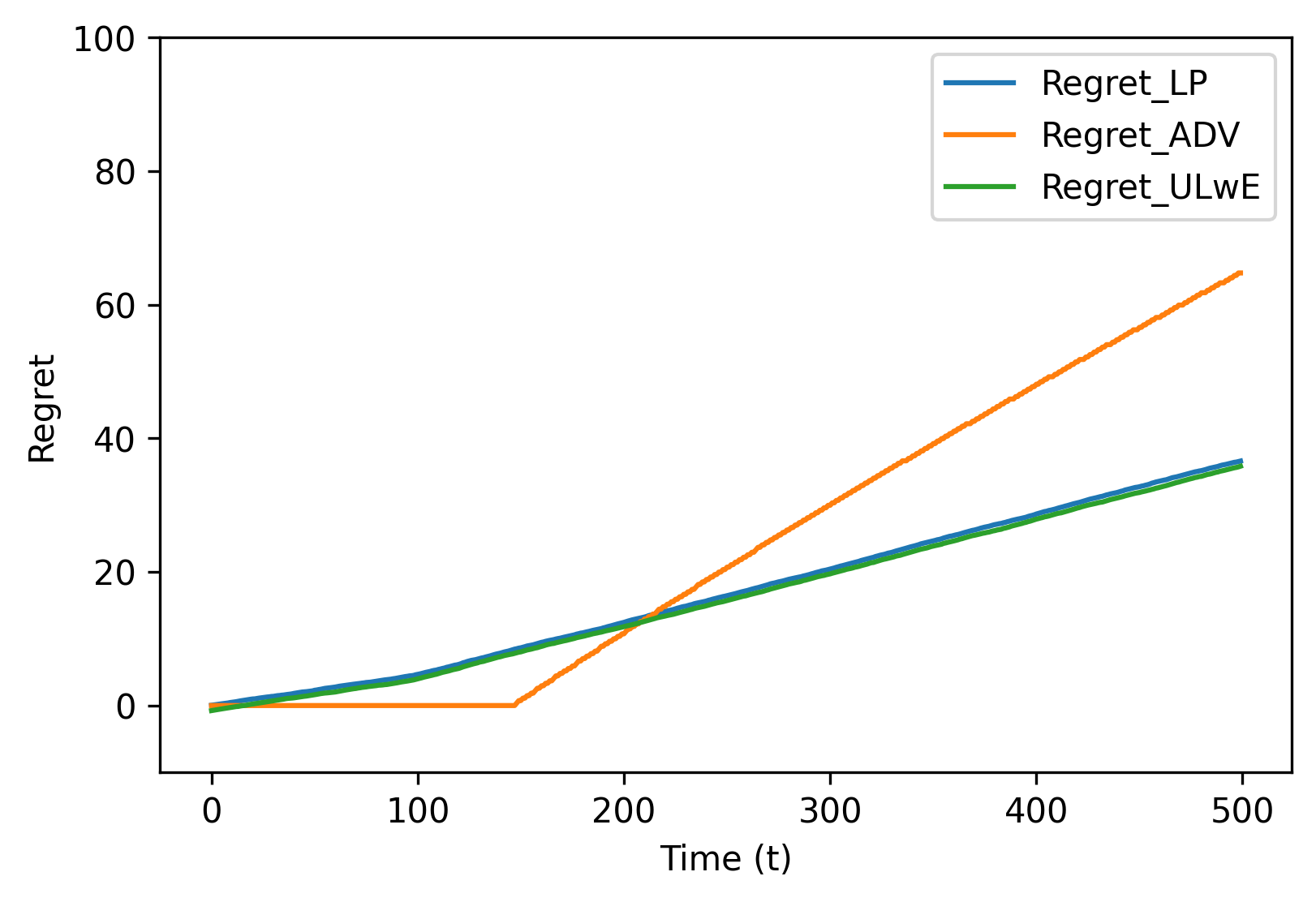}}
\caption{Regret over Time under IID Arrival}
\label{fig:IID1}
\end{center}

\end{figure}

\cref{fig:IID1} illustrates the regret trajectories over a total time period $T$ of 500. The ADV algorithm (yellow line) initially maintains a stable regret, suggesting an optimal initial resource allocation. However, as time advances, an increase in regret is observed, a characteristic trait of the algorithm's greedy nature. This increase indicates a shift from optimal decisions to suboptimal ones due to the depletion of resources and consequent reduction in modified revenue. In contrast, the $\text{ALG}_{\text{LP}}$ and ULwE algorithms (blue and green lines, respectively) exhibit a more gradual and consistent increase in regret. The near-overlapping of these lines indicates a similarity in their performance, with both algorithms showing a steady increase in revenue regret over time.

At $T=500$, $\text{ALG}_{\text{ADV}}$ has the highest cumulative regret, suggesting that the heuristic greedy algorithm encounters difficulties in adapting its resource allocation strategies effectively under IID arrivals. In contrast, both $\text{ALG}_{\text{LP}}$ and the ULwE Algorithm exhibit similar and improved performance compared to $\text{ALG}_{\text{ADV}}$, showing effective adaptation to IID conditions with lower cumulative regret.

\subsection{Results under Adversarial Arrivals}

In this subsection, we examine the algorithms' performance under nonstationary arrival scenarios, where customer arrival probabilities change over time. Specifically, we simulate a scenario with Customer Type A and B having varying arrival rates across different phases of the time period $T$, to reflect changing customer behaviors.
\begin{figure}[ht]
\begin{center}
\centerline{\includegraphics[width=\columnwidth]{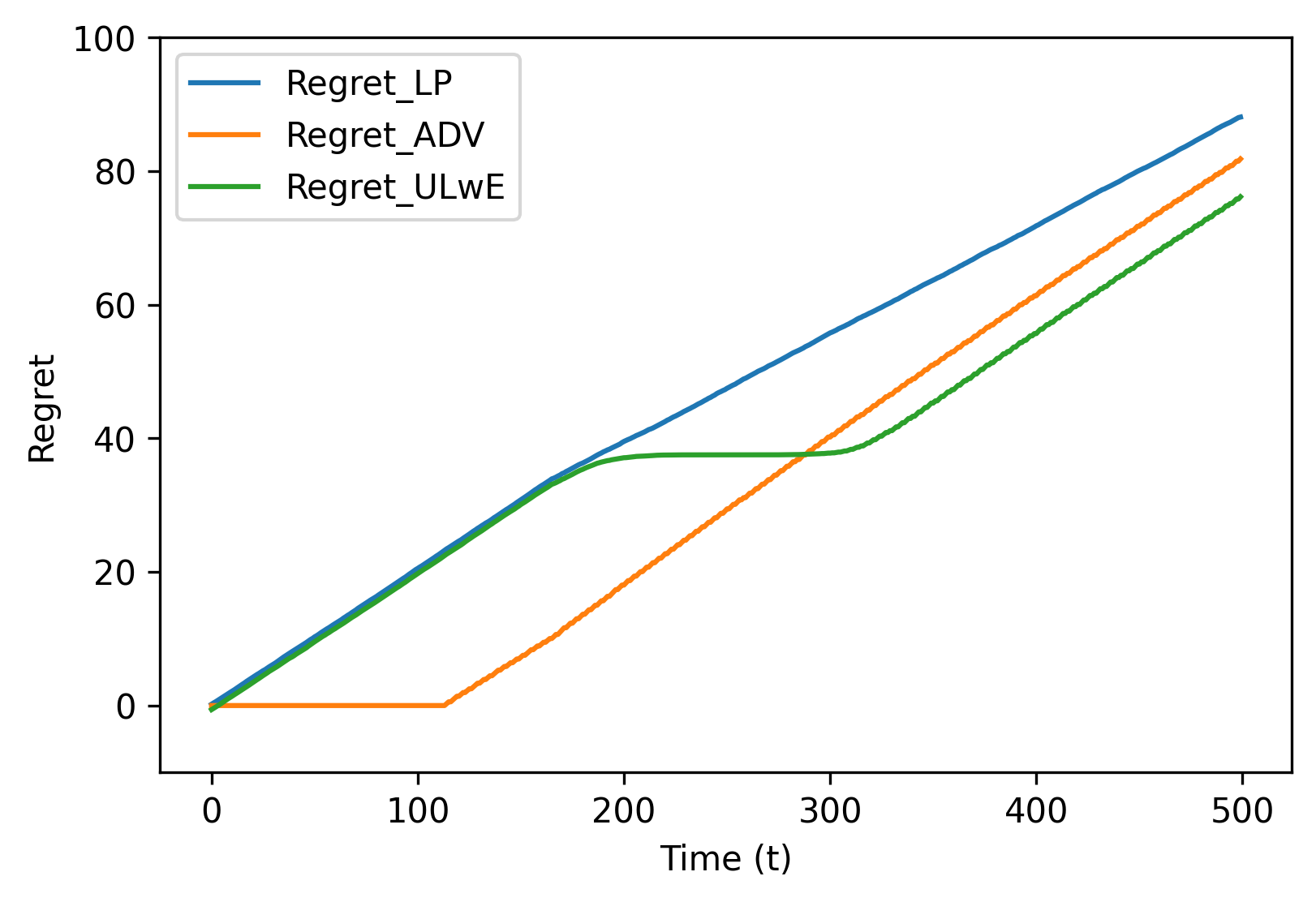}}
\caption{Regret over Time under Adversarial Arrival (ADV1)}
\label{fig:ADV1}
\end{center}
\end{figure}
In our analysis of nonstationary environments, as shown in \cref{fig:ADV1}, the regret trajectories of $\text{ALG}_{\text{LP}}$ (blue line) and $\text{ALG}_{\text{ADV}}$ (yellow line) were consistent with their performance under IID conditions. $\text{ALG}_{\text{LP}}$ exhibited a steady increase in regret over time, while $\text{ALG}_{\text{ADV}}$ showed initial stability followed by a sharp rise in the later stages. Interestingly, under nonstationary conditions, the regret of $\text{ALG}_{\text{ADV}}$ did not surpass that of $\text{ALG}_{\text{LP}}$, indicating better performance in dynamic settings.

The ULwE Algorithm (green line) displayed a distinct pattern. It initially follows the trend of  $\text{ALG}_{\text{LP}}$, but after a switch point at around $1/3T$, it adopts a pattern similar to $\text{ALG}_{\text{ADV}}$—initially stable, then sharply increasing, indicating its ability to integrate the strengths of both the LP and ADV approaches. 



\subsection{Results under General Arrivals}
\label{sec:more results}
\begin{figure}[ht]
\begin{center}
\centerline{\includegraphics[width=\columnwidth]{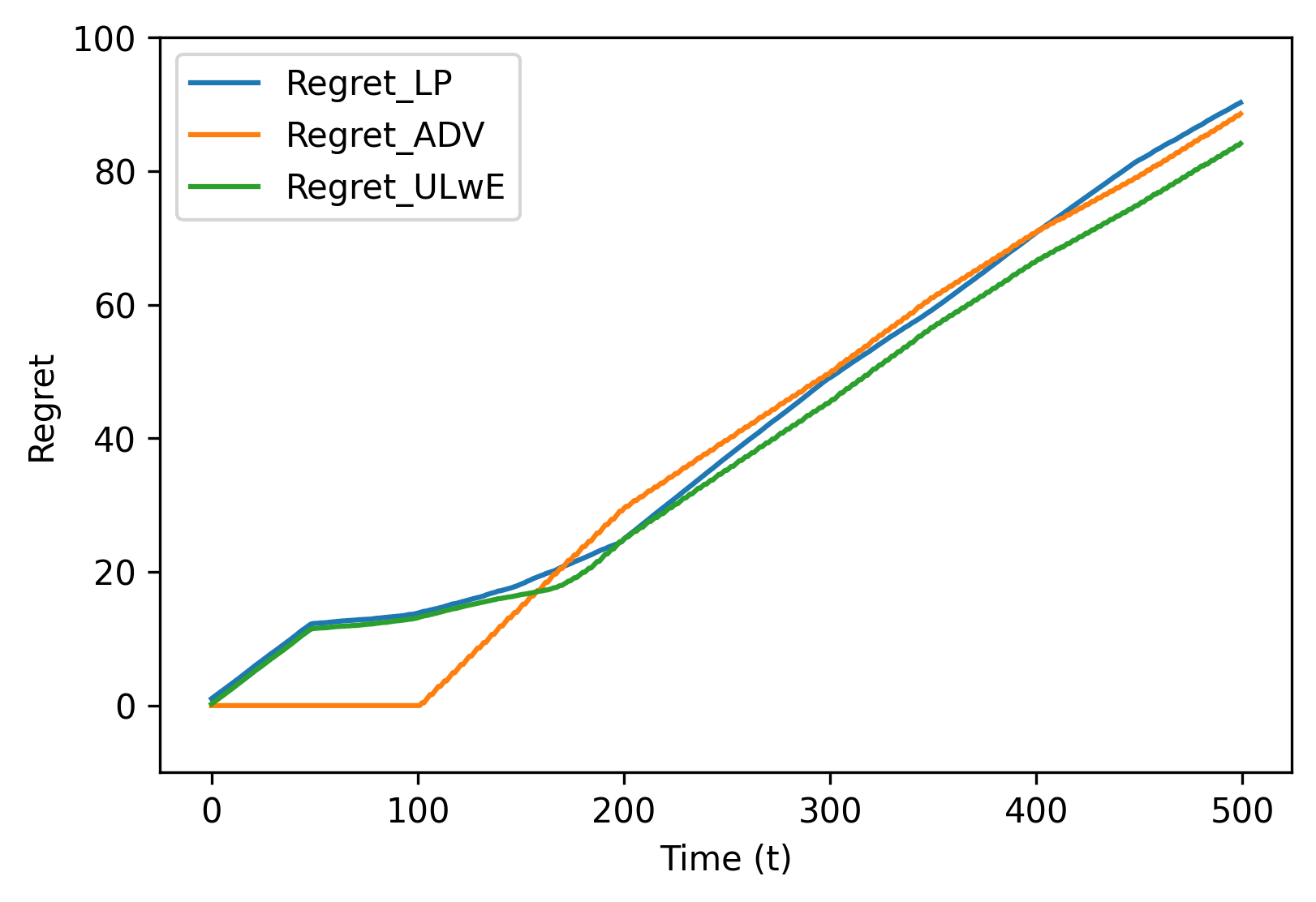}}
\caption{Regret over Time under Adversarial Arrival (ADV2)}
\label{fig:ADV2}
\end{center}
\end{figure}
This subsection evaluates the performance of each algorithms under various customer arrival scenarios (shown in \cref{tab:arrival-rates}), which vary in nonstationarity levels. The focus is primarily on the ADV2 setting, representing high nonstationarity. Key findings underscore the ULwE algorithm's robustness, especially in low nonstationarity settings like ADV1, due to its adaptive switching between $\text{ALG}_{\text{LP}}$ and $\text{ALG}_{\text{ADV}}$. Although its performance dips in more nonstationary situations (e.g., ADV2), it still outperforms in stationary contexts.

In summary, the $\text{ALG}_{\text{ADV}}$ maintains steady performance in nonstationary environments, while the $\text{ALG}_{\text{LP}}$, limited by its static resource allocation, struggles with variable arrival rates and higher nonstationarity. ULwE's adaptability in resource allocation and response to changing conditions minimizes regret and underscores its effectiveness in managing nonstationary uncertainties.

\section{Conclusion}
\label{sec: conclusion}
In conclusion, we proposed the ULwE Algorithm based on CBwK to address the resource allocation problem under nonstationary environments. Our algorithm leverages contextual information to make informed decisions and adaptively balances exploration and exploitation to accommodate rapidly changing customer preferences. By assuming no prior knowledge of the per-period arrival process and considering the variability in arrival probabilities, our algorithm overcomes the limitations of stationary environments and achieves efficient resource allocation.

While bandit algorithms have been previously used for resource allocation problems, our approach extends the existing literature by explicitly accounting for the variability in customer purchase probabilities. Compared to related works, our algorithm achieves a regret bound of approximately $\tilde{O}(\sqrt{n|\Theta|T})$ in the case of near-i.i.d. arrivals, providing sublinear regret and a constant CR under nonstationary arrivals. In addition to the theoretical analysis, our experiments compared the ULwE Algorithm with $\text{ALG}_{\text{LP}}$ and $\text{ALG}_{\text{ADV}}$. The results consistently demonstrated that the ULwE Algorithm outperforms these algorithms, achieving lower regret under nonstationary arrival patterns. Future research could explore the applicability of the algorithm in other domains and further enhance its performance in dynamic environments.

\nocite{langley00}

\bibliography{main}
\bibliographystyle{icml2024}

\newpage
\appendix
\onecolumn
\section{Appendix: LP formulations}

To capture the maximum achievable total reward, we define a static LP problem, denoted as $\text{LP}(\theta^*)$:
\begin{align}
\label{eq:LP1}
\begin{split}
\text{LP}(\theta^*) = \max_{s_{ij}, i\in [n], j\in[L]} & \sum_{i \in [n]} r_i \sum_{j \in [L]} \lambda_j s_{ij} f_i(x^{(j)}, \theta_i^*)\\
\text{s.t. } & \sum_{j \in [L]}\lambda_j  s_{ij} f_i(x^{(j)}, \theta_i^*) \leq c_i, \ \forall i \in [n]\\
& \sum_{i\in [n]}s_{ij}=1,~\forall j\in [L]\\
&  s_{ij} \geq 0, \ \forall i \in [n], j \in [L].
\end{split}
\end{align}
The coefficients in the LP are determined by the latent variables $\theta^*$. The objective of this LP is to maximize the total reward obtained by assigning resources to customers based on their features and latent variables. The LP is subject to constraints that ensure consistency between budget limitations and resource assignment. The optimal solution to this LP is denoted as $s^*$.

We denote $\bar U$ the maximum optimal objective value of $U^t$, i.e.,
\begin{align}
\label{eq:LP3}
\begin{split}
\bar U = \max&_{s_{ij}, \theta_i\in \Theta_i, i\in [n], j\in[L]}  \sum_{i \in [n]} r_i \sum_{j \in [L]}\lambda_j s_{ij} \bar  f_i (x^{(j)}, \theta_i)\\
\text{s.t. } & \sum_{j \in [L]} \lambda_j s_{ij}  \bar f_i (x^{(j)}, \theta_i) \leq c_i, \ \forall i \in [n]\\
& \sum_{i\in [n]}s_{ij}=1,~\forall j\in [L]\\
&  s_{ij} \geq 0, \ \forall i \in [n], j \in [L].
\end{split}
\end{align}

\section{Appendix: Proof of the Deterministic Upper Bound on the Optimal Revenue}
\label{appendix: benchmark}
In this subsection, our primary objective is to establish a theoretical upper bound on the optimal revenue that can be achieved. This involves a comprehensive analysis of the optimal revenue (denoted as $\overline{\text{OPT}}$), which serves as a benchmark against which the performance of various algorithms can be measured. By exploring the upper limits of achievable revenue, we aim to provide a clear framework for evaluating the efficiency and effectiveness of different resource allocation strategies. This analysis not only aids in quantifying the regret incurred by these algorithms but also offers insights into the potential for revenue maximization under varying operational constraints and market conditions.

Firstly, the resource allocation problem is formulated as follows: At time zero, the system has capacities $c$ of $n$ items and a finite time $t>0$ to assign them. 
 
Let $\sum_{j=1}^L  {\mu_j}^s f_I(x^{j^s},\theta^*) )$ denote the probability of item $I$ is purchased by customers to time $s$. The term \( {\mu_j}^t \) denotes the probability of encountering a customer of type \( j \) at time \( t \), reflecting the non-stationary nature of customer arrivals. The function \( f_I(x^{j^t},\theta^*) \) captures the probability of a customer of type \( j \) with feature vector \( x^{j^t} \) purchasing resource \( I \), given the latent variable \( \theta^* \) associated with the resource. A consumption is realized at time $s$ if $\sum_{j=1}^L  {\mu_j}^s f_I(x^{j^s},\theta^*) )=1$, in which case the system sells one item and receives revenue of $r_I$.

We introduce an indicator variable \( \delta_{I,s} \), which equals 1 if resource \( I \) is assigned at time \( s \), and 0 otherwise. The class of all non-anticipating policies, denoted by \( \Pi \), must satisfy
\[
\sum_{s=1}^t \sum_{j=1}^L  {\mu_j}^s f_I(x^{j^s},\theta^*) \delta_{I,s} \leq c_I 
\]

In addressing the challenge of optimizing resource allocation under constraints of customer behavior and time, we formulate a function that encapsulates the expected revenue over a continuous time horizon. The decision-making policy \( \pi \) guides the allocation of resources, determining each resource \( I \) being dynamically chosen based on customer features and the characteristics of the resources, as indicated by \( \delta_{I,s} \).

Given a allocation policy $\pi \in \Pi$, an initial capacity $c>0$, and a sales horizon $t>0$, we denote the expected revenue by
$$
J_{\pi}(c, t) \doteq \mathrm{E}_\pi\left[  \sum_{s=1}^t r_I \sum_{j=1}^L {\mu_j}^s f_I(x^{j^s},\theta^*)\delta_{I,s}  \right]
$$

Here, \( r_I \) represents the revenue earned from resource \( I \) upon a customer's purchase, as determined by policy \( \pi \) and the assignment indicated by $\delta_{I,s}$. 
The problem is to find a decision-making policy ${\pi}^*$ (if one exists) that maximizes the total expected revenue generated over $[0, t]$, denoted $\text{OPT}$. Equivalently,
$$
\text{OPT} \doteq \sup _{\pi \in \Pi} J_{\pi}(c, t)
$$

This formulation allows for a dynamic and adaptive approach to resource allocation, where the decision policy \( \pi \) can be optimized based on the varying probabilities of customer types and their purchasing behaviors over time. The objective $\text{OPT}$ thus represents the total expected revenue, taking into account the fluctuating customer landscape and the inherent uncertainties in customer preferences and behaviors.

In our pursuit to devise an optimal strategy for resource allocation in a customer interaction system, we propose a deterministic linear programming model, denoted as \( J^D \). This model is designed to maximize the deterministic revenue while adhering to the constraints of resource capacities and customer arrival patterns. 
The objective function aims to maximize the total deterministic revenue overall resources, customer types, and time periods.
$$
\begin{aligned}
J^D(c,t) = \max_{s_{ij}, i\in [n], j\in[L]} & \sum_{i=1}^n \sum_{j=1}^L \sum_{s=1}^t r_i  s_{ij}^s f_i(x^j,\theta^*) \\
 \text{s.t. } &\sum_{j=1}^L \sum_{s=1}^t {\mu_j}^s s_{ij}^s f_i(x^j,\theta^*) \leq c_i, \forall i \in [n] \\
& \sum_{i=1}^n\sum_{s=1}^t s_{ij}^s =  \sum_{s=1}^t {\mu_j}^s, \forall j \in [L] \\
& s_{ij}^s \geq 0 ,\forall i \in [n], \forall j \in [L]
\end{aligned}
$$ 
All the variables and parameters in this LP are deterministic. They are either constants known a priori or decision variables whose values are to be determined by solving the LP without any randomness involved. The first constraint ensures that the expected number of resource allocations does not exceed the capacity \( c_i \) of each resource \( i \). This constraint is crucial for maintaining a balance between maximizing revenue and not overcommitting the available resources. The second constraint
guarantees that the total allocations for each customer type \( j \) over the time horizon equal the expected number of arrivals \( \lambda_j \) of that customer type. This aligns the resource allocation with the anticipated customer demand.

Finally, the non-negativity constraint ensures that the decision variables remain feasible, reflecting the reality that negative resource allocation is not possible.

\begin{proposition}
\label{prop:benchmark}
     For all $0 \leq c<+\infty$ and $0 \leq t<+\infty$,

$$
\text{OPT} \leq J^{D}(c, t)
$$
\end{proposition}
\begin{proof}
Initially, the problem is formulated as finding the best policy $ \pi$ to maximize the expected sum of returns, subject to a cost constraint:
$$
\begin{aligned}
\text{OPT} \doteq \max_{\pi}  & \quad 
 \mathrm{E}_\pi\left[  \sum_{s=1}^t r_I \sum_{j=1}^L {\mu_j}^s f_I(x^{j^s},\theta^*)\delta_{I,s}  \right]\\
\text{s.t. }&\sum_{s=1}^t 
 \sum_{j=1}^L  {\mu_j}^s f_I(x^{j^s},\theta^*) ) \leq c_I 
\end{aligned}
$$
Here we introduce binary decision variables $\pi_{ij}$ to represent the allocation of resources to each customer at each time point $s>0$. We add a constraint $\sum_{i=1}^n \pi_{ij}^s = 1$, which represents each customer can only be assigned one item at each time point. The problem thus becomes a binary integer programming problem:

$$
\begin{aligned}
\text{OPT} = \max_{\pi_{ij}}  &  \sum_{i=1}^n \sum_{j=1}^L \sum_{s=1}^t r_i  {\mu_j}^s \pi_{ij}^s f_i(x^{j^s},\theta^*)  \\
\text{s.t. }&\sum_{s=1}^t 
 \sum_{j=1}^L  {\mu_j}^s \pi_{ij}^s f_i(x^{j^s},\theta^*) ) \leq c_i \quad \forall i \in [n]\\
 & \sum_{i=1}^n \pi_{ij}^s =1\\
 & \pi_{ij}^s = \{0,1\}
\end{aligned}
$$
To establish the relationship between the original optimization problem \( \text{OPT} \) and \( J^{D}(c, t) \), we begin to get a new LP problem $OPT_1$ by relaxing the binary constraint on \( \pi_{ij}^s \), allowing it to take any value between 0 and 1. This relaxation naturally leads to \( OPT_1 \geq \text{OPT} \), as \( OPT_1 \) includes all feasible solutions of \( \text{OPT} \) and potentially more. We then introduce a new variable \( m_{ij}^s = {\mu_j}^s \pi_{ij}^s \), noting that since \( {\mu_j}^s \) is also between 0 and 1, it follows that \( 0 \leq m_{ij}^s \leq 1 \). 

The first constraint of \( OPT_1 \), \(\sum_{s=0}^t \sum_{j=1}^L  m_{ij}^s f_i(x^{j^s},\theta^*) ) \leq c_i \), is tighter than the first constraint of \( J^{D}(c, t) \), \( \sum_{s=0}^t \sum_{j=1}^L  {\mu_j}^s s_{ij}^s f_i(x^{j^s},\theta^*) ) \leq c_i \), as \( \sum_{s=0}^t \sum_{j=1}^L  m_{ij}^s f_i(x^{j^s},\theta^*) ) \) is larger than \( \sum_{s=0}^t \sum_{j=1}^L  {\mu_j}^s s_{ij}^s f_i(x^{j^s},\theta^*) ) \). 

The second constraint of \( OPT_1 \), $\sum_{i=1}^n m_{ij}^s = \sum_{i=1}^n {\mu_j}^s \pi_{ij}^s = {\mu_j}^s$, applies to each time period, making it more stringent than the overall time constraint in \( J^{D}(c, t) \),$\sum_{i=1}^n\sum_{s=1}^t s_{ij}^s =  \sum_{s=1}^t {\mu_j}^s$. Consequently, we conclude that \( \text{OPT} \leq OPT_1 \leq J^{D}(c, t) \), completing the proof. 
\end{proof}

Having rigorously established the inequality $\text{OPT} \leq J^{D}(c, t)$, we have effectively demonstrated that $J^{D}(c, t)$ serves as an upper bound on the optimal revenue, which we denote as $\overline{\text{OPT}}$ 
 . This upper bound is crucial as it provides a benchmark against which the performance of various algorithms can be measured. In light of this, we are now positioned to define the concept of 'regret' in our context. Specifically, regret can be quantitatively expressed as the difference between the upper bound of the optimal revenue and the actual optimal revenue achieved, formulated as $$\text{Regret} = \overline{\text{OPT}} - \text{ALG}$$
This definition of regret is instrumental in evaluating the efficiency of different algorithms. By measuring how closely an algorithm's performance approaches the theoretical upper limit of revenue, we can assess its effectiveness in resource allocation and decision-making under various operational scenarios.

\section{Appendix: Technical Lemmas}

\begin{lemma}
\label{lemma:azuma}
(Azuma-Hoeffding inequality) Let $\mathcal{G}_0\subset \mathcal{G}_1\subset \ldots \subset \mathcal{G}_n$ be a filtration $\mathbf{G}$, and $X_0,\ldots,X_n$ a martingale associated with $\mathbf{G}$, and $|X_i-X_{i-1}|\leq c_i$, $\forall~ i=1,\ldots,n$ almost surely. Then, it holds that
\begin{equation*}
\mathbb{P}\left[\left|X_n-X_0\right|> \epsilon\right] \leq 2\exp\left(-\frac{\epsilon^2}{2\sum_{k=1}^nc_k^2 }\right),
\end{equation*}
for all $\epsilon>0$.
\end{lemma}

\begin{lemma}
\label{lemma:freedman}
(\cite{badanidiyuru2014resourceful})  Let $\mathcal{G}_0\subset \mathcal{G}_1\subset \ldots \subset \mathcal{G}_n$ be a filtration, and $X_1,\ldots,X_n$ be real random variables such that $X_i$ is $\mathcal{G}_i$-measurable, $\mathbb{E}[X_i|\mathcal{G}_{i-1}]=0$ and $|X_i|\leq c$ for all $i\in [n]$ and some $c>0$. Then with probability at least $1-\delta$ it holds that
\begin{equation*}
\left|\sum_{i=1}^n X_i\right|  \leq \sqrt{4\sum_{i=1}^n\mathbb{E}[X_i^2|\mathcal{G}_{i-1}]\log(2n/\delta)+5c^2\log^2(2n/\delta)}.
\end{equation*}
\end{lemma}

\begin{proof}
First notice for all $i\in [n],j\in [L],t\in [T]$, we have 
\begin{align*}
\mathbb{E} \left[\mu_j \bar s_{ij}^t \bar f_i(x^{(j)},\Omega_i^{t-1})|\mathcal{F}_{t-1}\right]= & \mathbb{E}\left[\mathbb{E}\left[ \mathbb{I}(i,j,t)|\mathcal{F}_{t-1}\right]\bar f_i(x^{(j)},\Omega_i^{t-1})|\mathcal{F}_{t-1} \right]\\
= & \mathbb{E}\left[\mathbb{I}(i,j,t)|\mathcal{F}_{t-1}\right] \mathbb{E}\left[\bar f_i(x^{(j)},\Omega_i^{t-1})|\mathcal{F}_{t-1} \right]\\
= & \mathbb{E}\left[\mathbb{I}(i,j,t)\bar f_i(x^{(j)},\Omega_i^{t-1})|\mathcal{F}_{t-1}\right],
\end{align*}

and $$\left|\sum_{j=1}^L (\mathbb{I}(i,j,t)-\mu_j\bar s_{ij}^t)\bar f_i(x^{(j)},\Omega_i^{t-1})\right|\leq 1.$$ Let $X_0=0$, $X_t-X_{t-1}=\sum_{j=1}^L (\mathbb{I}(i,j,t)-\mu_j\bar s_{ij}^t)\bar f_i(x^{(j)}, \Omega_i^{t-1})$, then we know $\mathbb{E}\left[X_t|\mathcal{F}_{t-1}\right]=0$ for all $t\in [T]$. To use \cref{lemma:freedman}, we need to estimate $V_T:=\sum_{t=1}^T\mathbb{E}\left[X_t^2|\mathcal{F}_{t-1}\right]$. To be specific, we have
 
 \begin{align*}
 V_T = & \sum_{t=1}^T \mathbb{E}\left[\left. \left( \sum_{j=1}^L (\mathbb{I}(i,j,t)-\mu_j\bar s_{ij}^t) \bar f_i(x^{(j)},\Omega_i^{t-1}) \right)^2\right| \mathcal{F}_{t-1} \right]\\
 = & \sum_{t=1}^T \mathbb{E}\left[\left.  \sum_{j=1}^L (\mathbb{I}(i,j,t)-\mu_j\bar s_{ij}^t)^2 \bar f_i(x^{(j)},\Omega_i^{t-1})^2 \right| \mathcal{F}_{t-1}\right]\\
 & + \sum_{t=1}^T \underbrace{\mathbb{E}\left[\left.  \sum_{j\neq j'\in [L]} (\mathbb{I}(i,j,t)-\mu_j\bar s_{ij}^t)(\mathbb{I}(i,j',t)-\mu_j\bar s_{ij'}^t) \bar f_i(x^{(j)},\Omega_i^{t-1}) \bar f_i(x^{(j')},\Omega_i^{t-1}) \right| \mathcal{F}_{t-1}\right]}_{\leq 0}\\
 \leq  & \sum_{t=1}^T \mathbb{E}\left[\left.  \sum_{j=1}^L (\mathbb{I}(i,j,t)-\mu_j\bar s_{ij}^t)^2 \bar f_i(x^{(j)},\Omega_i^{t-1})^2 \right| \mathcal{F}_{t-1}\right]\\
 = & \sum_{t=1}^T\left( \mathbb{E}\left[\left.\sum_{j=1}^L \mathbb{I}(i,j,t)^2\bar f_i(x^{(j)},\Omega_i^{t-1})^2\right|\mathcal{F}_{t-1}\right]  - 2\mathbb{E}\left[\left.\sum_{j=1}^L\mathbb{I}(i,j,t)\mu_j \bar s_{ij}^t \bar f_i(x^{(j)},\Omega_i^{t-1}) \right|\mathcal{F}_{t-1}\right]\right)\\
 & + \sum_{t=1}^T \mathbb{E}\left[\left.\mu_j^2  (\bar s_{ij}^t)^2 \bar f_i(x^{(j)},\Omega_i^{t-1})^2 \right|\mathcal{F}_{t-1}\right] \\
 \leq & \sum_{t=1}^T\left( \mathbb{E}\left[\left.\sum_{j=1}^L \mathbb{I}(i,j,t)\bar f_i(x^{(j)},\Omega_i^{t-1})\right|\mathcal{F}_{t-1}\right]  - 2\mathbb{E}\left[\left.\sum_{j=1}^L\mathbb{I}(i,j,t)\mu_j \bar s_{ij}^t \bar f_i(x^{(j)},\Omega_i^{t-1}) \right|\mathcal{F}_{t-1}\right]\right)\\
 & + \sum_{t=1}^T\mathbb{E}\left[\left.\mu_j^2 (\bar s_{ij}^t)^2 \bar f_i(x^{(j)},\Omega_i^{t-1}) \right|\mathcal{F}_{t-1}\right] \\
 = & \sum_{t=1}^T\mathbb{E}\left[ \left. \sum_{j=1}^L \mu_j\bar s_{ij}^t(1-\mu_j\bar s_{ij}^t)\bar f_i(x^{(j)},\Omega_i^{t-1}) \right|\mathcal{F}_{t-1} \right]\\
 \leq & \sum_{t=1}^T\mathbb{E}\left[ \left. \sum_{j=1}^L \mu_j\bar s_{ij}^t\bar f_i(x^{(j)},\Omega_i^{t-1}) \right|\mathcal{F}_{t-1} \right].
 \end{align*}
 The first inequality is due to the following calculation: given any $j\neq j'\in [L]$,
 \begin{align*}
 & \mathbb{E}\left[\left.(\mathbb{I}(i,j,t)-\mu_j\bar s_{ij}^t)(\mathbb{I}(i,j',t)-\mu_j\bar s_{ij'}^t) \bar f_i(x^{(j)},\Omega_i^{t-1}) \bar f_i(x^{(j')},\Omega_i^{t-1}) \right|\mathcal{F}_{t-1} \right] \\
 = & (\mu_j\bar s_{ij}^t(1-\mu_j\bar s_{ij}^t)(-\mu_{j'}\bar s_{ij'}^t)+\mu_{j'}\bar s_{ij'}^t(-\mu_j\bar s_{ij}^t)(1-\mu_{j'}\bar s_{ij'}^t) \\
 + &(\mu_j(1-\bar s_{ij}^t)+\mu_{j'}(1-\bar s_{ij'}^t)+(1-\mu_j-\mu_{j'}) )\mu_j\bar s_{ij}^t\mu_{j'}\bar s_{ij'}^t) \cdot\bar f_i(x^{(j)},\Omega_i^{t-1}) \bar f_i(x^{(j')},\Omega_i^{t-1}) \\
 = &-\mu_j\bar s_{ij}^t\mu_{j'}\bar s_{ij'}^t\bar f_i(x^{(j)},\Omega_i^{t-1}) \bar f_i(x^{(j')},\Omega_i^{t-1})\leq 0.
 \end{align*}
 So, by using \cref{lemma:freedman}, for any $0<\delta<1$, we obtain with probability at least $1-\delta$,
 \begin{align*}
  & \left|\sum_{t=1}^T\sum_{j=1}^L \left( \mathbb{I}(i,j,t)-\mu_j \bar s_{ij}^t \right) \bar f_i(x^{(j)}, \Omega_i^{t-1})  \right|  \\
  & \leq\sqrt{4\sum_{t=1}^T\mathbb{E}\left[ \left. \sum_{j=1}^L \mu_j\bar s_{ij}^t\bar f_i(x^{(j)},\Omega_i^{t-1}) \right|\mathcal{F}_{t-1} \right]\log(2T/\delta)+5\log^2(2T/\delta) }\\
 & \leq \sqrt{4\sum_{t=1}^T\mathbb{E}\left[ \left. \sum_{j=1}^L \mu_j\bar s_{ij}^t\bar f_i(x^{(j)},\Omega_i^{t-1}) \right|\mathcal{F}_{t-1} \right]\log(2T/\delta)}+\sqrt{5}\log(2T/\delta),
 \end{align*}
 where in the second inequality we use the fact $\sqrt{a+b}\leq \sqrt{a}+\sqrt{b}$ for any $a,b,>0$. Therefore, we conclude with probability at least $1-\delta$,
 \begin{align*}
& \sum_{i=1}^n\left|\sum_{t=1}^T\sum_{j=1}^L \left( \mathbb{I}(i,j,t)-\mu_j \bar s_{ij}^t \right) \bar f_i(x^{(j)},\Omega_i^{t-1})  \right|\\
 & \leq   \sum_{i=1}^n\sqrt{4\sum_{t=1}^T\mathbb{E}\left[ \left. \sum_{j=1}^L \mu_j\bar s_{ij}^t\bar f_i(x^{(j)},\Omega_i^{t-1}) \right|\mathcal{F}_{t-1} \right]\log(2T/\delta)}+\sqrt{5}n\log(2T/\delta)\\
 & \leq\sqrt{4n\sum_{i=1}^n\sum_{t=1}^T\mathbb{E}\left[ \left. \sum_{j=1}^L \mu_j\bar s_{ij}^t\bar f_i(x^{(j)},\Omega_i^{t-1}) \right|\mathcal{F}_{t-1} \right]\log(2T/\delta)}+\sqrt{5}n\log(2T/\delta)\\
&\leq \sqrt{4nT\log(2T/\delta)}+\sqrt{5}n\log(2T/\delta),
 \end{align*}
 where in the second inequality we use the Cauchy-Schwarz inequality.
\end{proof}

\begin{corollary}
\label{cor:condition}
After the modification in the removal process of $\theta$,
\begin{equation*}
\mathbb{P}[\mathcal{E}]\geq 1-(1+\log T)\beta(n,T).
\end{equation*}
\end{corollary}
\cref{cor:condition} is directly turned out by \cref{prop:high prob ucb}.

\begin{lemma}
\label{lemma:azuma2} For any $0<\delta<1$ and $t\in [T]$, it holds that
\begin{equation*}
\mathbb{P}\left[\forall~\theta\in\Theta: \left|\sum_{l=1}^t\sum_{i=1}^n r_i(s_{iJ^l}(\theta)-\bar s_{iJ^l}^l -\sum_{j=1}^L\mu_j^l(s_{ij}(\theta)-\bar s_{ij}^l)) f_{ij}(x^{(j)},\theta)   \right|>\max\limits_{i\in [n]}r_i \sqrt{8t\log(2|\Theta|/\delta)} \right]\leq \delta.
\end{equation*}
\end{lemma}

\begin{proof}
First note $\mathbb{E}[(s_{iJ^t}(\theta)-\bar s_{iJ^t}^t) f_{iJ^t}(x^{(J^t)},\theta) |\mathcal{F}_{t-1}]=\sum_{j=1}^L\mu_j^t(s_{ij}(\theta) -\bar s_{ij}^t )f_{i}(x^{(J^t)},\theta)$ and \\ $\left|\sum_{i=1}^n r_i\left(s_{iJ^l}(\theta)-\bar s_{iJ^t}^t -\sum_{j=1}^L\mu_j^t(s_{ij}(\theta)-\bar s_{ij}^t)\right) f_{i}(x^{(j)},\theta) \right|\leq 2\max\limits_{i\in [n]} r_i$ for all $t\in [T]$, so we define $X_0=0, X_l-X_{l-1}=\sum_{i=1}^n r_i\left(s_{iJ^l}(\theta)-\bar s_{iJ^l}^l -\sum_{j=1}^L\mu_j^l(s_{ij}(\theta)-\bar s_{ij}^l)\right) f_{i}(x^{(j)},\theta)$ so that $X_0,\ldots, X_t$ is a martingale associated with the filtration $\mathcal{F}_0,\ldots,\mathcal{F}_{t-1}$. So by using the Azuma-Hoeffding inequality (c.f. \cref{lemma:azuma}) we have for any $\theta\in \Theta$,
\begin{equation*}
\mathbb{P}\left[ \left|\sum_{l=1}^t\sum_{i=1}^n r_i\left(s_{iJ^l}(\theta)-\bar s_{iJ^l}^l -\sum_{j=1}^L\mu_j^l(s_{ij}(\theta)-\bar s_{ij}^l)\right) f_{i}(x^{(j)},\theta)   \right|>\epsilon\right]\leq 2e^{-\epsilon^2/(8t(\max\limits_{i\in [n]} r_i )^2 )}.
\end{equation*}
By letting $\delta/|\Theta|=2e^{-\epsilon^2/(8t(\max\limits_{i\in [n]} r_i )^2 )}$ and using the union bound we obtain the desired result.
\end{proof}

\begin{lemma}
\label{lemma:azuma3} For any $0<\delta<1$ and $t\in [T]$, it holds that
\begin{equation*}
\mathbb{P}\left[ \left|\sum_{l=1}^t \left(\bar s_{iJ^l}^lf_{i}(x^{(J^l)},\Omega^{l-1}) - \sum_{j=1}^L \mu_j \bar s_{ij}^lf_{i}(x^{(j)},\Omega^{l-1})    \right) \right|> \sqrt{2t\log(2/\delta)} \right]\leq \delta.
\end{equation*}
\end{lemma}

\section{Appendix: Omitted Proofs from \cref{section:algo}}
\label{Asec: regret proof}
\subsection{Proofs under Stationary Arrivals}

We point out the expected regret in period $t$ is (for now ignore events at the end of the horizon)
\[ \sum_{i=1}^n r_i \sum_{j=1}^L \mu_j(s^*_{ij} - \bar s_{ij}^t) f_i(x^{(j)},\theta_i^*).\]
The following proposition states that the regret can be bounded from above by the size of the confidence intervals of purchase probabilities in period $t$.
\begin{proposition}
\label{prop:f upper bound}
Suppose in each period $t$, if $\bar f_i (x^{(j)}, \Omega_i^{t-1}) \geq f_i(x^{(j)},\theta_i^*)$ for all $i \in [n], j \in [L]$, then 
\[ \sum_{i=1}^n r_i \sum_{j=1}^L \mu_j(s^*_{ij} - \bar s_{ij}^t) f_i(x^{(j)},\theta_i^*) \leq \sum_{i=1}^n r_i \sum_{j=1}^L \mu_j \bar s_{ij}^t ( \bar f_i (x^{(j)}, \Omega_i^{t-1})-  f_i(x^{(j)},\theta_i^*)).\]
\end{proposition}

\begin{proof}
\begin{align*}
& \sum_{i=1}^n r_i \sum_{j=1}^L \mu_j(s^*_{ij} - \bar s_{ij}^t) f_i(x^{(j)},\theta_i^*)\\
= &\sum_{i=1}^n r_i \sum_{j=1}^L \mu_j s^*_{ij} f_i(x^{(j)},\theta_i^*) -  \sum_{i=1}^n r_i\left[ \sum_{j=1}^L \mu_j \bar s_{ij}^t \bar f_i (x^{(j)}, \Omega_i^{t-1}) -  \sum_{j=1}^L \mu_j \bar s_{ij}^t ( \bar f_i (x^{(j)}, \Omega_i^{t-1})-  f_i(x^{(j)},\theta_i^*)) \right]\\
= & LP(\theta^*) - U^t + \sum_{i=1}^n r_i \sum_{j=1}^L \mu_j \bar s_{ij}^t ( \bar f_i (x^{(j)}, \Omega_i^{t-1})-  f_i(x^{(j)},\theta_i^*))\\
\leq & \sum_{i=1}^n r_i \sum_{j=1}^L \mu_j \bar s_{ij}^t ( \bar f_i (x^{(j)}, \Omega_i^{t-1})-  f_i(x^{(j)},\theta_i^*)).
\end{align*}

The last inequality follows from the condition that the values of ``purchase probabilities'' in LP (\ref{eq:LP2}) are at least those in LP (\ref{eq:LP1}).
\end{proof}

Proposition \ref{prop:f upper bound} provides an alternative perspective on the regret incurred in period $t$. Specifically, the algorithm selects "arm" $i$ under context $x^{(j)}$ with probability $\mu_j \bar s_{ij}^t$. In this case, the algorithm experiences regret equal to $\bar f_i (x^{(j)}, \Omega_i^{t-1}) - f_i(x^{(j)},\theta_i^*)$.

We next provide a direct result from the Azuma-Hoeffding inequality (see \cref{lemma:azuma}). 

\begin{lemma}
\label{lemma:azuma inequ}
For any $0<\delta<1$, $t\in [T]$, suppose all $I^t,J^t$ are given, 
\begin{equation*}
\mathbb{P}\left[\left|\sum_{l=1}^t  (f_{I^l}(x^{(J^l)},\theta_{I^l}^*) -a^{l})\right|> \sqrt{2 t\log(2/\delta) } \right]\leq \delta.
\end{equation*}
\end{lemma}

\begin{proof}
First note $\mathbb{E}\left[  a^{l} |\mathcal{F}_{l-1} \right]=f_{I^l}(x^{l},\theta_{I^l}^*)$ and $|f_{I^l}(x^{(J^l)},\theta^*_{I^l})-a^l|\leq1$ for all $l\in [t]$. Let $X_0=0, X_l-X_{l-1}= f_{I^l}(x^{l},\theta_{I^l}^*)-a^l$ for all $l\in [t]$ we know $X_0,\ldots, X_l$ is a martingale associated with the filtration $\mathcal{F}_0,\mathcal{F}_1,\ldots,\mathcal{F}_{l-1}$. Thus using the Azuma-Hoeffding inequality (c.f. Lemma \ref{lemma:azuma}) we have for any $\epsilon>0$,
\begin{equation*}
\mathbb{P}\left[\left|\sum_{l=1}^t r_{I^t}(f_{I^l}(x^{l},\theta_{I_l}^*) -a^{l})\right|>\epsilon \right]\leq 2e^{-\epsilon^2/(2t)}.
\end{equation*}
Then the desired result is obtained by setting $\epsilon$ such that $\delta = 2e^{-\epsilon^2/(2t)}$.
\end{proof}

Proposition \ref{prop:high prob ucb} shows the probabilistic event $\mathcal{E}:= \forall~ i\in [n],j\in [L], t\in [T] : \bar{f}_i(x^{(j)},\Omega_i^{t-1}) \geq f_i(x^{(j)},\theta_i^*)$ indeed happens with high probability. The event $\mathcal{E}$ ensures that the maximum purchase probability for resource $i$ based on a set $\Omega$ of valid latent variables is greater than the true probability. Once the probabilistic event $\mathcal{E}$ has been established to occur with high probability, analyzing the revenue regret under different conditions becomes more convenient. 

\begin{proposition}
\label{prop:high prob ucb}
$\mathbb{P}\left[ \mathcal{E}   \right]\geq 1-(1+\log T)\beta(n,T)$.
\end{proposition}
\begin{proof}
 Since 
\begin{equation*}
\mathbb{P}\left[ \forall~ i\in [n],j\in [L], t\in [T] : \bar{f}_i(x^{(j)},\Omega_i^{t-1}) \geq f_i(x^{(j)},\theta_i^*)   \right]\geq \mathbb{P}\left[  \forall~ i\in [n], t\in [T]: \theta_i^*\in \Omega_{i}^t   \right],
\end{equation*}
so it is sufficient to lower bound $ \mathbb{P}\left[  \forall~ i\in [n], t\in [T]: \theta_i^*\in \Omega_{i}^t   \right]$.

Consider the probabilistic event 

\begin{equation*}
A(t)=\left\{\left| \sum_{l=1}^t f_{I^l}(x^{(J^l)},\theta_{I^l}^*)-a^l \right|> \sqrt{2 t\log(2t/\beta(n,T))} \right\}.
\end{equation*}

We know from \cref{lemma:azuma inequ} that 
\begin{equation*}
\mathbb{P}\left[ A(t) \right]\leq \beta(n,T)/t.
\end{equation*}

From step 3 of the algorithm and then using the union bound, we have
\begin{align*}
\mathbb{P}\left[  \forall~ i\in [n], t\in [T]: \theta_i^*\in \Omega_{i}^t   \right]= & \mathbb{P}\left[\left|\sum_{t' \in D_{I^{t}}^t(\theta_i^*)}f_{I^{t'}}(x^{t'},\theta_i^*) - a^{t'} \right| \leq \sqrt{2t\log(2t/\beta(n,T))} \right] \\
\geq & \mathbb{P}\left[\left| \sum_{l=1}^t f_{I^l}(x^{(J^l)},\theta_{I^l}^*)-a^l \right|\leq \sqrt{2 t\log(2t/\beta(n,T))} \right]\\
\geq & 1-  \mathbb{P}\left[\cup_{t=1}^T A(t)   \right] \\
\geq & 1- \sum_{t=1}^T \beta(n,T)/t\\
\geq & 1-(\log T +1)\beta(n,T),
\end{align*}
where in the last inequality we use the fact $\sum_{t=1}^T\frac{1}{t}\leq \log T +1$.

\end{proof}
Proposition \ref{Aprop:revenue regret bound} establishes an upper bound on the expected regret resulting from the deviation of revenue. It quantifies the relationship between the regret and the revenue shortfall caused by the algorithm's decisions. By considering this upper bound, we can gain insights into the potential regret incurred due to revenue variations.

\begin{proposition}
\label{Aprop:revenue regret bound}
We have
\begin{align*}
& \mathbb{E}\left[\underbrace{\sum_{t=1}^T\sum_{i=1}^n r_i \sum_{j=1}^L \mu_j (s_{ij}^*-\bar{s}_{ij}^t) f_i(x^{(j)},\theta_i^*)}_{(*)}   \right] \\ 
\leq &  \max\limits_{i\in [n]}r_i \left(\sqrt{ n \max\limits_{i\in [n]}|\Theta_i|}+1\right)\sqrt{8T \log(2T/\beta(n,T))} + \bar U /T \beta(n,T) + LP(\theta^*)\beta(n,T)(\log T+1).
\end{align*}
\end{proposition}

\begin{proof}
From Proposition \ref{prop:f upper bound}, given $\mathcal{E}$, it is sufficient to upper bound the term 
\begin{align*}
& \underbrace{\sum_{t=1}^T\sum_{i=1}^n r_i \sum_{j=1}^L \mu_j \bar s_{ij}^t ( \bar f_i (x^{(j)}, \Omega_i^{t-1})-  f_i(x^{(j)},\theta_i^*))}_{(**)} \\
= & \left.\sum_{t=1}^T \mathbb{E}\left[ \sum_{i=1}^n r_i \sum_{j=1}^L \mathbb{I}(i,j,t) ( \bar f_i (x^{(j)}, \Omega_i^{t-1})-  f_i(x^{(j)},\theta_i^*))\right| \mathcal{F}_{t-1} \right]\\
= &  \left.\sum_{t=1}^T \mathbb{E}\left[r_{I^t}   ( \bar f_{I^t} (x^{(J^t)}, \Omega_{I^t}^{t-1})-  f_{I^t}(x^{(J^t)},\theta_{I^t}^*))\right| \mathcal{F}_{t-1} \right]\\
\leq  & \max\limits_{i\in [n]} r_i\cdot  \left.\mathbb{E}\left[   \underbrace{\sum_{t=1}^T \left( \bar f_{I^t} (x^{(J^t)}, \Omega_{I^t}^{t-1})- a^t \right)}_{(***)}+\underbrace{\sum_{t=1}^T \left(  a^t-  f_{I^t}(x^{(J^t)},\theta_{I^t}^*)\right)}_{(****)} \right| \mathcal{F}_{t-1}\right],
\end{align*}
where the first inequality holds because of $\mathbb{E}[\mathbb{I}(i,j,t)|\mathcal{F}_t]=\mu_t\bar s_{ij}^t$ and law of iterated expectation. Next, from step 3 of the algorithm and the Cauchy-Schwarz inequality, it holds almost surely
\begin{align*}
|(***)| = & \left|\sum_{i=1}^n\sum_{\theta_i\in \Theta_i}\sum_{t'\in D_i^T( \theta_i)}\left(  f_{i} (x^{(J^{t'})}, \theta_i)- a^{t'}    \right)\right|\\
\leq &   \sum_{i=1}^n \sum_{\theta_i\in \Theta_i} \sqrt{2|D_i^T(\theta_i)|\log(2T/\beta(n,T))}\\
\leq &  \sqrt{n\max\limits_{i\in [n]}|\Theta_i|\sum_{i=1}^n\sum_{\theta_i \in \Theta_i} 2|D_i^T(\theta_i)|\log(2T/\beta(n,T)) } \\
= & \sqrt{2 n\max\limits_{i\in [n]}|\Theta_i|T \log(2T/\beta(n,T))}.
\end{align*}

Moreover, from \ref{lemma:azuma inequ} we know
\begin{equation*}
\mathbb{P}\left[|(****)|\leq \sqrt{2T \log(2T/\beta(n,T))}\right]\geq 1-\beta(n,T)/T.
\end{equation*}

Therefore, let $\alpha= \max\limits_{i\in [n]}r_i\left(\sqrt{ n \max\limits_{i\in [n]}|\Theta_i|}+1\right)\sqrt{8T \log(2T/\beta(n,T))}$, we obtain
\begin{equation*}
\mathbb{P}[(**) \leq \alpha|\mathcal{E}]\geq 1-\beta(n,T)/T.
\end{equation*}

Therefore, noticing $\mathbb{P}[(**)> \alpha,\mathcal{E}]\leq \mathbb{P}[(**)> \alpha|\mathcal{E}]\leq \beta(n,T)/T$,

\begin{align*}
\mathbb{E}[(*)] = & \mathbb{E}[(*)|\mathcal{E}]\mathbb{P}[\mathcal{E}]+\mathbb{E}[(*)|\mathcal{E}^c]\mathbb{P}[\mathcal{E}^c]\\
\leq & \mathbb{E}[(**)|\mathcal{E}] + LP(\theta^*)\beta(n,T)(\log T+1)\\
= & \mathbb{E}[(**)|(**)\leq \alpha,\mathcal{E}]\mathbb{P}[(**)\leq \alpha,\mathcal{E}] + \mathbb{E}[(**)|(**)> \alpha,\mathcal{E}]\mathbb{P}[(**)> \alpha,\mathcal{E}]\\
& + LP(\theta^*)\beta(n,T)(\log T +1)\\
\leq & \alpha +  \bar U \beta(n,T)/T  + LP(\theta^*)\beta(n,T)(\log T+1).
\end{align*}

\end{proof}

\cref{Alemma:freedman inequ} provides a crucial concentration estimate in analyzing the regret incurred by the violation of the inventory constraints.

\begin{lemma}
\label{Alemma:freedman inequ}

For any $0<\delta<1$, it holds that,
\begin{equation*}
\mathbb{P}\left[ \sum_{i=1}^n \left|\sum_{t=1}^T\sum_{j=1}^L \left( \mathbb{I}(i,j,t)-\mu_j \bar s_{ij}^t \right) \bar f_i(x^{(j)},\Omega_i^{t-1})  \right|>\sqrt{4nT\log(2T/\delta)}+\sqrt{5}n\log(2T/\delta) \right]\leq \delta.
\end{equation*}
\end{lemma}

\begin{proof}
First notice for all $i\in [n],j\in [L],t\in [T]$, we have 
\begin{align*}
\mathbb{E} \left[\mu_j \bar s_{ij}^t \bar f_i(x^{(j)},\Omega_i^{t-1})|\mathcal{F}_{t-1}\right]= & \mathbb{E}\left[\mathbb{E}\left[ \mathbb{I}(i,j,t)|\mathcal{F}_{t-1}\right]\bar f_i(x^{(j)},\Omega_i^{t-1})|\mathcal{F}_{t-1} \right]\\
= & \mathbb{E}\left[\mathbb{I}(i,j,t)|\mathcal{F}_{t-1}\right] \mathbb{E}\left[\bar f_i(x^{(j)},\Omega_i^{t-1})|\mathcal{F}_{t-1} \right]\\
= & \mathbb{E}\left[\mathbb{I}(i,j,t)\bar f_i(x^{(j)},\Omega_i^{t-1})|\mathcal{F}_{t-1}\right],
\end{align*}

and $$\left|\sum_{j=1}^L (\mathbb{I}(i,j,t)-\mu_j\bar s_{ij}^t)\bar f_i(x^{(j)},\Omega_i^{t-1})\right|\leq 1.$$ Let $X_0=0$, $X_t-X_{t-1}=\sum_{j=1}^L (\mathbb{I}(i,j,t)-\mu_j\bar s_{ij}^t)\bar f_i(x^{(j)}, \Omega_i^{t-1})$, then we know $\mathbb{E}\left[X_t|\mathcal{F}_{t-1}\right]=0$ for all $t\in [T]$. To use Lemma \ref{lemma:freedman}, we need to estimate $V_T:=\sum_{t=1}^T\mathbb{E}\left[X_t^2|\mathcal{F}_{t-1}\right]$. To be specific, we have
 
 \begin{align*}
 V_T = & \sum_{t=1}^T \mathbb{E}\left[\left. \left( \sum_{j=1}^L (\mathbb{I}(i,j,t)-\mu_j\bar s_{ij}^t) \bar f_i(x^{(j)},\Omega_i^{t-1}) \right)^2\right| \mathcal{F}_{t-1} \right]\\
 = & \sum_{t=1}^T \mathbb{E}\left[\left.  \sum_{j=1}^L (\mathbb{I}(i,j,t)-\mu_j\bar s_{ij}^t)^2 \bar f_i(x^{(j)},\Omega_i^{t-1})^2 \right| \mathcal{F}_{t-1}\right]\\
 & + \sum_{t=1}^T \underbrace{\mathbb{E}\left[\left.  \sum_{j\neq j'\in [L]} (\mathbb{I}(i,j,t)-\mu_j\bar s_{ij}^t)(\mathbb{I}(i,j',t)-\mu_j\bar s_{ij'}^t) \bar f_i(x^{(j)},\Omega_i^{t-1}) \bar f_i(x^{(j')},\Omega_i^{t-1}) \right| \mathcal{F}_{t-1}\right]}_{\leq 0}\\
 \leq  & \sum_{t=1}^T \mathbb{E}\left[\left.  \sum_{j=1}^L (\mathbb{I}(i,j,t)-\mu_j\bar s_{ij}^t)^2 \bar f_i(x^{(j)},\Omega_i^{t-1})^2 \right| \mathcal{F}_{t-1}\right]\\
 = & \sum_{t=1}^T\left( \mathbb{E}\left[\left.\sum_{j=1}^L \mathbb{I}(i,j,t)^2\bar f_i(x^{(j)},\Omega_i^{t-1})^2\right|\mathcal{F}_{t-1}\right]  - 2\mathbb{E}\left[\left.\sum_{j=1}^L\mathbb{I}(i,j,t)\mu_j \bar s_{ij}^t \bar f_i(x^{(j)},\Omega_i^{t-1}) \right|\mathcal{F}_{t-1}\right]\right)\\
 & + \sum_{t=1}^T \mathbb{E}\left[\left.\mu_j^2  (\bar s_{ij}^t)^2 \bar f_i(x^{(j)},\Omega_i^{t-1})^2 \right|\mathcal{F}_{t-1}\right] \\
 \leq & \sum_{t=1}^T\left( \mathbb{E}\left[\left.\sum_{j=1}^L \mathbb{I}(i,j,t)\bar f_i(x^{(j)},\Omega_i^{t-1})\right|\mathcal{F}_{t-1}\right]  - 2\mathbb{E}\left[\left.\sum_{j=1}^L\mathbb{I}(i,j,t)\mu_j \bar s_{ij}^t \bar f_i(x^{(j)},\Omega_i^{t-1}) \right|\mathcal{F}_{t-1}\right]\right)\\
 & + \sum_{t=1}^T\mathbb{E}\left[\left.\mu_j^2 (\bar s_{ij}^t)^2 \bar f_i(x^{(j)},\Omega_i^{t-1}) \right|\mathcal{F}_{t-1}\right] \\
 = & \sum_{t=1}^T\mathbb{E}\left[ \left. \sum_{j=1}^L \mu_j\bar s_{ij}^t(1-\mu_j\bar s_{ij}^t)\bar f_i(x^{(j)},\Omega_i^{t-1}) \right|\mathcal{F}_{t-1} \right]\\
 \leq & \sum_{t=1}^T\mathbb{E}\left[ \left. \sum_{j=1}^L \mu_j\bar s_{ij}^t\bar f_i(x^{(j)},\Omega_i^{t-1}) \right|\mathcal{F}_{t-1} \right].
 \end{align*}
 The first inequality is due to the following calculation: given any $j\neq j'\in [L]$,
 \begin{align*}
 & \mathbb{E}\left[\left.(\mathbb{I}(i,j,t)-\mu_j\bar s_{ij}^t)(\mathbb{I}(i,j',t)-\mu_j\bar s_{ij'}^t) \bar f_i(x^{(j)},\Omega_i^{t-1}) \bar f_i(x^{(j')},\Omega_i^{t-1}) \right|\mathcal{F}_{t-1} \right] \\
 = & (\mu_j\bar s_{ij}^t(1-\mu_j\bar s_{ij}^t)(-\mu_{j'}\bar s_{ij'}^t)+\mu_{j'}\bar s_{ij'}^t(-\mu_j\bar s_{ij}^t)(1-\mu_{j'}\bar s_{ij'}^t) \\
 + &(\mu_j(1-\bar s_{ij}^t)+\mu_{j'}(1-\bar s_{ij'}^t)+(1-\mu_j-\mu_{j'}) )\mu_j\bar s_{ij}^t\mu_{j'}\bar s_{ij'}^t) \cdot\bar f_i(x^{(j)},\Omega_i^{t-1}) \bar f_i(x^{(j')},\Omega_i^{t-1}) \\
 = &-\mu_j\bar s_{ij}^t\mu_{j'}\bar s_{ij'}^t\bar f_i(x^{(j)},\Omega_i^{t-1}) \bar f_i(x^{(j')},\Omega_i^{t-1})\leq 0.
 \end{align*}
 So, by using Lemma \ref{lemma:freedman}, for any $0<\delta<1$, we obtain with probability at least $1-\delta$,
 \begin{align*}
  & \left|\sum_{t=1}^T\sum_{j=1}^L \left( \mathbb{I}(i,j,t)-\mu_j \bar s_{ij}^t \right) \bar f_i(x^{(j)}, \Omega_i^{t-1})  \right|  \\
  & \leq\sqrt{4\sum_{t=1}^T\mathbb{E}\left[ \left. \sum_{j=1}^L \mu_j\bar s_{ij}^t\bar f_i(x^{(j)},\Omega_i^{t-1}) \right|\mathcal{F}_{t-1} \right]\log(2T/\delta)+5\log^2(2T/\delta) }\\
 & \leq \sqrt{4\sum_{t=1}^T\mathbb{E}\left[ \left. \sum_{j=1}^L \mu_j\bar s_{ij}^t\bar f_i(x^{(j)},\Omega_i^{t-1}) \right|\mathcal{F}_{t-1} \right]\log(2T/\delta)}+\sqrt{5}\log(2T/\delta),
 \end{align*}
 where in the second inequality we use the fact $\sqrt{a+b}\leq \sqrt{a}+\sqrt{b}$ for any $a,b,>0$. Therefore, we conclude with probability at least $1-\delta$,
 \begin{align*}
& \sum_{i=1}^n\left|\sum_{t=1}^T\sum_{j=1}^L \left( \mathbb{I}(i,j,t)-\mu_j \bar s_{ij}^t \right) \bar f_i(x^{(j)},\Omega_i^{t-1})  \right|\\
 & \leq   \sum_{i=1}^n\sqrt{4\sum_{t=1}^T\mathbb{E}\left[ \left. \sum_{j=1}^L \mu_j\bar s_{ij}^t\bar f_i(x^{(j)},\Omega_i^{t-1}) \right|\mathcal{F}_{t-1} \right]\log(2T/\delta)}+\sqrt{5}n\log(2T/\delta)\\
 & \leq\sqrt{4n\sum_{i=1}^n\sum_{t=1}^T\mathbb{E}\left[ \left. \sum_{j=1}^L \mu_j\bar s_{ij}^t\bar f_i(x^{(j)},\Omega_i^{t-1}) \right|\mathcal{F}_{t-1} \right]\log(2T/\delta)}+\sqrt{5}n\log(2T/\delta)\\
&\leq \sqrt{4nT\log(2T/\delta)}+\sqrt{5}n\log(2T/\delta),
 \end{align*}
 where in the second inequality we use the Cauchy-Schwarz inequality.
\end{proof}

Proposition \ref{Aprop:inventory regret bound} calculates the regret incurred from exceeding the inventory constraints. Denote $(\cdot)^+=\max\{\cdot,0\}.$ This proposition allows us to understand the relationship between resource allocation decisions and the incurred regret due to inventory limitations.

\begin{proposition}
\label{Aprop:inventory regret bound}
We have
\begin{align*}
&\mathbb{E}\left[\underbrace{\sum_{i=1}^n r_i \left[ \sum_{t=1}^T\sum_{j=1}^L\mathbb{I}(i,j,t) a^t -c_i \right]^+}_{(*)}  \right]\\
&\leq \sqrt{16\max\limits_{i\in n}|\Theta_i|nT\log(2T/\beta(n,T))}+\sqrt{5}n\log(2T/\beta(n,T))+n(T\log T + 2T)\beta(n,T). 
\end{align*}
\end{proposition}

\begin{proof}
We rewrite 
\begin{align*}
(*)  =& \sum_{i=1}^n r_i\Bigg[ \underbrace{\sum_{t=1}^T\sum_{j=1}^L \mathbb{I}(i,j,t)(a^t-\bar f_i(x^{(j)},\Omega_i^{t-1}))}_{(**)} \\
&  +\underbrace{\sum_{t=1}^T\sum_{j=1}^L(\mathbb{I}(i,j,t)-\mu_j\bar s_{ij}^t)\bar f_i(x^{(j)},\Omega_i^{t-1})}_{(***)} +\underbrace{\sum_{t=1}^T\sum_{j=1}^L\mu_j\bar s_{ij}^t \bar f_i(x^{(j)},\Omega_i^{t-1})-c_i}_{(****)}  \Bigg]^+\\
\leq&  \sum_{i=1}^n r_i|(**)|+\sum_{i=1}^nr_i|(***)|+\sum_{i=1}^nr_i[(****)]^+.
\end{align*}

From step 3 of the algorithm and the Cauchy-Scharz inequality, we know given $\mathcal{E}$,
\begin{align*}
\sum_{i=1}^n|(**)| = & \sum_{i=1}^n \left|\sum_{\theta_i\in \Theta_i}\sum_{t'\in D_i^T(\theta_i)} \mathbb{I}(i,J^t,t')  \left( a^{t'} - \bar f_i(x^{(J^{t'})},\bar \theta^{t'}) \right) \right|\\
\leq & \sum_{i=1}^n\sum_{\theta_i\in \Theta_i}\sqrt{2|D_i^T(\theta_i)|\log(2T/\beta(n,T))}\\
\leq & \sqrt{2n\max\limits_{i\in [n]} |\Theta_i|T\log(2T/\beta(n,T)) }
\end{align*}
holds almost surely.  From \cref{Alemma:freedman inequ} we know
\begin{equation*}
\mathbb{P}\left[ \sum_{i=1}^n |(***)|\leq \sqrt{4nT\log(2T/\beta(n,T))}+\sqrt{5}n\log(2T/\beta(n,T))   \right] \geq 1-\beta(n,T).
\end{equation*}
Recall $\bar s_{ij}^t$ is the optimal solution of LP (\ref{eq:LP2}). Thus, noticing $ \lambda_j=T\mu_j $, it holds for all $i\in [n]$ almost surely
\begin{equation*}
\sum_{j=1}^L \mu_j \bar{s}_{ij}^t \bar f_i(x^{(j)},\Omega_i^{t-1})\leq c_i/T.
\end{equation*}
Hence, we know
$(****)\leq 0 $holds almost surely. 

Therefore, by setting $\alpha=\max\limits_{i\in [n]}r_i \sqrt{16\max\limits_{i\in [n]}|\Theta_i|nT\log(2T/\beta(n,T))}+\sqrt{5}n\log(2T/\beta(n,T))$, we obtain
\begin{equation*}
\mathbb{P}[(*)\leq \alpha|\mathcal{E}]\geq 1-\beta(n,T).
\end{equation*}

Thus, noticing

\begin{align*}
\mathbb{P}[((*)\leq \alpha,\mathcal{E})^c] = & 1-\mathbb{P}[(*)\leq \alpha ,\mathcal{E}]\\
= & 1 - \mathbb{P}[(*)\leq \alpha|\mathcal{E}]\mathbb{P}[\mathcal{E}]\\
\leq & 1 - (1-\beta(n,T))(1-(\log T+1)\beta(n,T))\\
\leq & (\log T+2)\beta(n,T),
\end{align*}
we then have
\begin{align*}
\mathbb{E}[(*)] = & \mathbb{E}[(*)|(*)\leq \alpha,\mathcal{E}]\mathbb{P}[(*)\leq \alpha,\mathcal{E}]+\mathbb{E}[(*)|((*)\leq  \alpha,\mathcal{E})^c ]\mathbb{P}[((*)\leq \alpha,\mathcal{E})^c]\\
\leq & \alpha + (\log T + 2)\beta(n,T) nT\\
= &  \max\limits_{i\in [n]}r_i\sqrt{16\max\limits_{i\in [n]}|\Theta_i|nT\log(2T/\beta(n,T))}+\sqrt{5}n\log(2T/\beta(n,T))+n(T\log T + 2T)\beta(n,T).
\end{align*}
\end{proof}
Theorem \ref{Athm:iid regret bound} provides a summary of the expected regret bound for the $\text{ALG}_{\text{LP}}$ algorithm. Specifically, the regret bound, denoted by $\tilde O\left(\sqrt{\max\limits{i\in [n]} |\Theta_i| n T}\right)$, indicates that the expected regret increases with the square root of the maximum parameter space size $|\Theta_i|$, the number of resources $n$, and the time horizon $T$. This regret bound is particularly relevant in scenarios where customer arrivals exhibit a near-IID pattern, and it allows us to assess the scalability and performance of the $\text{ALG}_{\text{LP}}$ in different settings and make informed decisions regarding resource allocation. 

\begin{theorem}
\label{Athm:iid regret bound}
The expected regret of the algorithm is upper bounded by 
\begin{equation*}
8\max\limits_{i\in [n]}r_i\sqrt{\max\limits_{i\in [n]}|\Theta_i|nT \log(2nT^2)} +\sqrt{5}n\log(2nT^2)+\bar U /(nT^2)+ LP(\theta^*)(\log T+1)/(nT) + \log T + 2.
\end{equation*}
\end{theorem}
\begin{proof}
The expected regret is upper bounded by the sum of the regret incurred by the expected revenue and violation of inventory constraints, that is,
\begin{equation*}
\mathbb{E}\left[\sum_{t=1}^T\sum_{i=1}^n r_i \sum_{j=1}^L \mu_j (s_{ij}^*-\bar{s}_{ij}^t) f_i(x^{(j)},\theta_i^*)   \right] + \mathbb{E}\left[\sum_{i=1}^n r_i\left[ \sum_{t=1}^T\sum_{j=1}^L\mathbb{I}(i,j,t) a^t-c_i \right]^+  \right].
\end{equation*}
So by piecing together Proposition \ref{Aprop:revenue regret bound}  and \ref{Aprop:inventory regret bound} we obtain the result.
\end{proof}
Moreover, a key feature of our regret bound is that it is independent of the number of contexts $L$, which means it might lead to efficient computation even when $L$ is very large, as long as we have a moderate $\max\limits_{i\in [n]}|\Theta_i|$. This feature is also shared by \cite{badanidiyuru2014resourceful}, where they have the cardinality of the policy space instead of the contextual space appear as a multiplicative factor in the regret bound (this allows them to treat resourceful contextual bandit problems with an infinite-dimensional contextual space).

\begin{remark}
\label{remark:regret}
When $\min\limits_{i\in [n]}c_i$ is very large, i.e., tends to infinity, it is easy to verify $$(1+\min\limits_{i\in [n]} c_i) \left(1-e^{-1/\min\limits_{i\in [n]}c_i }\right)\rightarrow 1$$, i.e., the CR has a limit $1+\frac{1}{1-1/e}$.
\end{remark}

\subsection{Proofs under adversarial arrivals}
Before analyzing the regret, we introduce an auxiliary online learning problem where we maintain the adversarial arrival setting but remove the inventory constraint that limits the consumption of each resource to at most $c_i$ units, for all $i \in [n]$. Suppose we have an algorithm $\Pi$ that decides $I^t$ in each period after observing $x^t$. The regret in this auxiliary problem, denoted as $\text{REG}_{\text{AUX}}$, is defined as the sum of regrets incurred in each period $t$:

\begin{equation*}
\text{REG}_{\text{AUX}}=\sum_{t\in [T]}(r_{I_*^t}f_{I_*^t }(x^t,\theta_{I_*^t}^*)-r_{I^t}f_{I^t }(x^t,\theta_{I^t}^*)  ),
\end{equation*}
where $I_*^t$ represents the optimal resource selection that maximizes the reward $r_i f_i(x^t, \theta_i^*)$ among all resources $i \in [n]$, denoted as $I_*^t:=\arg\max\limits_{i\in [n]} r_i f_i(x^t,\theta_i^*)$. By studying the regret in this auxiliary problem, we can gain insights into the impact of resource allocation decisions on the incurred regret, without considering the inventory constraint. 

Our analysis is based on a crucial finding in \cite{wangchi2022} (c.f., Lemma \ref{lemma:wangwill}), which says as long as we have an online algorithm for the auxiliary problem with low time regret, we can adapt the same algorithm to the complete problem within the same time regret plus some approximation error (unrelated to $T$).

\begin{lemma}
\label{lemma:wangwill}
(\cite{wangchi2022}) Let $\Pi$ be some algorithm that incurs regret 
$\text{REG}_{\text{AUX}}$ for the auxiliary online learning problem (without inventory constraints), suppose we apply the same algorithm $\Pi$ to the complete online learning problem (with inventory constraints), just with $r_i$ in period $t$ changed to $r_i^t =r_i\left( 1-\Psi\left( \frac{N_i^{t-1}}{c_i} \right) \right)$, and denote $\text{OPT}$ as the optimal revenue associated with some arrival sequence $x_1,\ldots,x_T$ and $\text{ALG}$ the revenue gained from running $\Pi$ with respect to the arrival sequence $x_1,\ldots,x_T$. Then, we have
\begin{equation*}
\text{OPT}\leq \frac{(1+\min\limits_{i\in [n]} c_i) \left(1-e^{-1/\min\limits_{i\in [n]}c_i }\right) }{1-1/e} \mathbb{E}[\textrm{ALG}]+\mathbb{E}[\text{REG}_{\text{AUX}}].
\end{equation*}
\end{lemma}

So, the analysis reduces to upper bound $\mathbb{E}[\text{REG}_{\text{AUX}}]$ properly. As shown in Proposition \ref{Aprop:auxiliary regret}, we derive $\mathbb{E}[\text{REG}_{\text{AUX}}]=\tilde{O}(\sqrt{nT})$.

\begin{proposition}
\label{Aprop:auxiliary regret}
 $\mathbb{E}[\text{REG}_{\text{AUX}}]\leq \max\limits_{i\in [n]}r_i(\sqrt{|\Theta|n}+1)\sqrt{2T \log(2T/\beta(n,T)) } + \max\limits_{i\in [n]} r_i (1/T+\log T+1)/n$.
\end{proposition}
\begin{proof}
Note given $\mathcal{E}$, it holds almost surely
\begin{align*}
\text{REG}_{\text{AUX}} = &  \sum_{t=1}^T (r_{I_*^t}f_{I_*^t}(x^{(J^t)},\theta_{I_*^t}^*)-r_{I^t}f_{I^t}(x^{(J^t)},\theta_{I^t}^*)) \\
= &  \sum_{t=1}^T (r_{I_*^t}f_{I_*^t}(x^{(J^t)},\theta_{I_*^t}^*)-r_{I^t}f_{I^t}(x^{(J^t)},\bar\theta^t)+r_{I^t}f_{I^t}(x^{(J^t)},\bar \theta^t)-r_{I^t}f_{I^t}(x^{(J^t)},\theta_{I^t}^*))\\
\leq &  \sum_{t=1}^T (r_{I_*^t}f_{I_*^t}(x^{(J^t)},\theta_{I_*^t}^*)-r_{I^t}f_{I^t}(x^{(J^t)},\bar\theta^t)+r_{I^t}f_{I^t}(x^{(J^t)},\bar \theta^t)-r_{I^t}f_{I^t}(x^{(J^t)},\theta_{I^t}^*))\\
\leq & \sum_{t=1}^T(r_{I^t}f_{I^t}(x^{(J^t)},\bar \theta^t)-r_{I^t}f_{I^t}(x^{(J^t)},\theta_{I^t}^*)) \\
\leq  & \max\limits_{i\in [n]}r_i \cdot \left(\underbrace{\sum_{t=1}^T(f_{I^t}(x^{(J^t)},\bar \theta^t)-a^t)}_{(*)}+\underbrace{\sum_{t=1}^T(a_t-f_{I^t}(x^{(J^t)},\theta_{I^t}^*))}_{(**)} \right),
\end{align*}
where in the first inequality we use the step 3 of the algorithm (i.e., $I^t=\arg\max\limits_{i\in [n]} r_i^t  \bar{f}_i(x^{(J^t)},\bar \theta^t)$), and in the second inequality we use the fact that $f_{I_*^t}(x^{(J^t)},\bar \theta^t)\geq f_{I_*^t}(x^t, \theta_{I_*^t}^*)$ almost surely given $\mathcal{E}$. From the step 4 of the algorithm, we know given $\mathcal{E}$, it holds almost surely
\begin{align*}
(*) = & \sum_{i=1}^n \sum_{\theta_i\in \Theta_i} \sum_{t'\in D_i^T(\theta_i) }(f_{I^{t'} }(x^{t'}, \theta_i)-a^{t'})\\
\leq & \sum_{i=1}^n\sum_{\theta_i\in \Theta_i}\sqrt{2D_i(\theta_i)\log(2T/\beta(n,T) )}\\
\leq &  \sqrt{2|\Theta|nT\log(2T/\beta(n,T) )},
\end{align*}
where in the last inequality we use the Cauchy-Schwarz inequality. Then, using \ref{lemma:azuma inequ} we know 
\begin{equation*}
\mathbb{P}\left[\left.(**)>  \sqrt{2T \log(2T/\beta(n,T)) }\right|\mathcal{E}\right]\leq \beta(n,T)/T.
\end{equation*}

Let $\alpha = \max\limits_{i\in [n]}r_i(\sqrt{|\Theta|n}+1)\sqrt{2T \log(2T/\beta(n,T)) }$, we know
\begin{equation*}
\mathbb{P}[\text{REG}_{\text{AUX}}\leq \alpha|\mathcal{E}]\geq 1-\beta(n,T)/T.
\end{equation*}
Hence,
\begin{align*}
\mathbb{E}[\text{REG}_{\text{AUX}}] = & \mathbb{E}[\text{REG}_{\text{AUX}}|\text{REG}_{\text{AUX}}\leq \alpha,\mathcal{E}]\mathbb{P}[\text{REG}_{\text{AUX}}\leq \alpha,\mathcal{E}]\\
& +\mathbb{E}[\text{REG}_{\text{AUX}}|(\text{REG}_{\text{AUX}}\leq \alpha,\mathcal{E})^c]\mathbb{P}[(\text{REG}_{\text{AUX}}\leq \alpha,\mathcal{E})^c]\\
\leq & \alpha + \max\limits_{i\in [n]} r_i T(1/T+\log T+1)\beta(n,T)\\
= &  \max\limits_{i\in [n]}r_i(\sqrt{|\Theta|n}+1)\sqrt{2T \log(2T/\beta(n,T)) } + \max\limits_{i\in [n]} r_i (1/T+\log T+1)/n.
\end{align*}

\end{proof}

Using Lemma \ref{lemma:wangwill} and Proposition \ref{Aprop:auxiliary regret} we obtain the main result of this section immediately.

\begin{theorem}
\label{Athm:adversarial regret}
It holds that
\begin{align*}
\text{OPT} \leq & \frac{(1+\min\limits_{i\in [n]} c_i) \left(1-e^{-1/\min\limits_{i\in [n]}c_i }\right) }{1-1/e} \mathbb{E}[\textrm{ALG}]\\
& +\max\limits_{i\in [n]}r_i(\sqrt{|\Theta|n}+1)\sqrt{2T \log(2T/\beta(n,T)) }+ \max\limits_{i\in [n]} r_i (1/T+\log T+1)/n.
\end{align*}
\end{theorem}

Theorem \ref{Athm:adversarial regret} provides a key result of $\text{ALG}_{\text{ADV}}$ regret upper bound under adversarial arrivals. Specifically, it establishes that the regret incurred by the algorithm exhibits a sublinear $T$ growth rate with a constant competitive ratio. Sublinear regret indicates that the corresponding algorithm's performance gradually approaches the performance level of the optimal offline policy as time progresses.The constant competitive ratio highlights the algorithm's performance relative to an optimal decision-making strategy. It provides a quantitative measure of the algorithm's effectiveness, indicating that it achieves regret that is at most a constant factor compared to the optimal strategy. Moreover, this result forms a crucial foundation for our subsequent analysis of the regret bound for the Unified Learning algorithm under nonstationary arrivals in Section \ref{Asec:hybrid}.

\subsection{Proofs under Nonstationary Arrivals}
\label{Asec:hybrid}
\subsubsection{Regret analysis: sublinear regret as long as $\text{ALG}_{\text{LP}}$ is not switched}
\label{Asec:ULwE regret IID1}
We first provide a proposition that states $\text{ALG}_{\text{LP}}$ always incurs a sublinear regret and a linear rate of resource consumption with high probability as long as the switch does not happen.
\begin{proposition}
\label{Aprop:not switch sublinear regret}
\label{Aprop:not switch regret}
Suppose $\text{ALG}_{\text{LP}}$ runs for $t$ iterations without being switched, then the expected regret up to time $t$ is upper bounded by
\begin{equation*}
16\sqrt{2|\Theta|nt\log(4|\Theta|t/\beta(n,T))}+\sqrt{5}n\log(2t/\beta(n,T))+(2/t+2+\log t)\beta(n,T)(LP(\theta^*)+nt).
\end{equation*}
Moreover, each resource $i\in [n]$ has at least
\begin{equation*}
\left[\frac{T-t}{T}c_i - \sqrt{8nt\log(2t/\beta(n,T))}\right]^+
\end{equation*}
remaining with probability at least $1-\beta(n,T)$.
\end{proposition}
\begin{proof}
Given $\mathcal{E}(t)$, we know the first condition in step 3 of the algorithm ensures
\begin{equation*}
\left| \sum_{l=1}^t \sum_{i=1}^n r_i (s_{iJ^l}^*-\bar s_{iJ^l}^l) f_{i}(x^{(J^l)},\theta^*) \right| \leq  \max\limits_{i\in [n]}r_i \sqrt{32t\log(4|\Theta|t/\beta(n,T)) }.
\end{equation*}
And by using \cref{lemma:azuma2} we derive with probability at least $1-\beta(n,T)/t$,
\begin{equation*}
\left| \sum_{l=1}^t \sum_{i=1}^n r_i\sum_{j=1}^L \mu^l_j (s_{ij}^*-\bar s_{ij}^l) f_{i}(x^{(j)},\theta^*) \right| \leq  \max\limits_{i\in [n]}r_i \sqrt{128t\log(4|\Theta|t/\beta(n,T)) }.
\end{equation*}
The second condition in step 3 of the algorithm ensures
\begin{equation*}
\forall~i\in [n]:  \sum_{l=1}^t \bar s_{iJ^l}^l\bar f_i(x^{(J^l)},\Omega^{l-1})   \leq    \frac{t}{T}c_i +\sqrt{2t\log(2t/\beta(n,T))}.
\end{equation*}
And by using \ref{lemma:azuma3} we derive with probability at least $1-\beta(n,T)/t$,
\begin{equation*}
\forall~i\in [n]:  \sum_{l=1}^t\sum_{j=1}^L \bar s_{ij}^l\bar f_i(x^{(j)},\Omega^{l-1})   \leq    \frac{t}{T}c_i +\sqrt{8nt\log(2t/\beta(n,T))}.
\end{equation*}
Therefore, using the same logic of proving Proposition \ref{Aprop:inventory regret bound}, we derive with probability at least $1-\beta(n,T)$,
\begin{equation*}
\sum_{i=1}^n\left[ \sum_{t=1}^T\sum_{j=1}^L\mathbb{I}(i,j,t) a^t -c_i \right]^+\leq \sqrt{64|\Theta|nt\log(2t/\beta(n,T))}+\sqrt{5}n\log(2t/\beta(n,T)).
\end{equation*}
Therefore, let $\alpha:=16\sqrt{2|\Theta|nt\log(4|\Theta|t/\beta(n,T))}+\sqrt{5}n\log(2t/\beta(n,T))$we have
\begin{equation*}
\mathbb{P}[\text{Regret}^t> \alpha|\mathcal{E}(t)]\leq (2/t+1)\beta(n,T).
\end{equation*}
So we end up with
\begin{align*}
\mathbb{E}[\text{Regret}^t] = &  \mathbb{E}[\text{Regret}^t|\text{Regret}^t\leq \alpha,\mathcal{E}(t)]\mathbb{P}[\text{Regret}^t\leq \alpha,\mathcal{E}(t)]\\
& +\mathbb{E}[\text{Regret}^t|(\text{Regret}^t\leq \alpha,\mathcal{E}(t))^c]\mathbb{P}[(\text{Regret}^t\leq \alpha,\mathcal{E}(t))^c]\\
\leq & \alpha + (2/t+2+\log t)\beta(n,T)(LP(\theta^*)+nt).
\end{align*}

\end{proof}

\subsubsection{Regret analysis: sublinear regret under IID arrivals}
\label{Asec:ULwE regret IID2}
In this section, we aim to demonstrate that under IID arrivals, the switch from $\text{ALG}_{\text{LP}}$ to $\text{ALG}_{\text{ADV}}$ does not occur with high probability. This can be achieved by showing that the conditions for the switch to happen are unlikely to be satisfied. Once we establish that the switch does not occur under IID arrivals, we can then apply Proposition \ref{Aprop:not switch sublinear regret} to prove that our algorithm still achieves a sublinear expected regret in this setting. 

To begin, we verify the first condition in step 3 of the algorithm is not violated for all $t\in [T]$ with high probability.

\begin{proposition}
\label{prop:iid first condition}
When $x^t$ arrives in a IID fashion, given $\mathcal{E}$, we have
\begin{align*}
& \mathbb{P}\left[\forall~t\in [T],\theta\in \Omega^{t-1}:\left|\sum_{l=1}^t\sum_{i=1}^n r_i (s_{iJ^l}(\theta)-\bar s_{iJ^l}^l)f_{i}(x^{(J^l)},\theta)  \right| \leq \max\limits_{i\in [n]}r_i \sqrt{32t\log(4|\Theta|t/\beta(n,T)) }  \right]\\
\geq & 1-(\log T+1)\beta(n,T).
\end{align*}
\end{proposition}

\begin{proof}
Using the same logic in Proposition \ref{prop:f upper bound}, we have for all $t\in [T]$,
\begin{equation*}
\sum_{i=1}^n r_i \sum_{j=1}^L \mu_j(s_{ij}(\theta)-\bar s_{ij}^t ) f_{i}(x^{(j)},\theta)\leq \sum_{i=1}^n r_i \sum_{j=1}^L \mu_j\bar s_{ij}^t (\bar f_i(x^{(j)},\Omega^{t-1} )-f_i(x^{(j)},\theta)  ).
\end{equation*}
So from the $\theta$-removal process of the algorithm we have,
\begin{align*}
\sum_{l=1}^t \sum_{i=1}^n r_i \sum_{j=1}^L \mu_j(s_{ij}(\theta)-\bar s_{ij}^l ) f_{i}(x^{(j)},\theta) \leq & \sum_{l=1}^t\mathbb{E}[r_{I^l}(\bar f_{I^l}(x^{(J^l)},\Omega^{l-1}) -f_{I^l}(x^{(J^l)},\theta) ) |\mathcal{F}_{l-1}]\\
\leq & \max\limits_{i\in [n]}r_i \sqrt{8t\log(2t/\beta(n,T)) }
\end{align*}
holds with probability at least $1-\beta(n,T)/(2t)$. Define probabilistic event 
\begin{equation*}
A(t)=\left\{\forall~\theta\in \Omega^{t-1}: \left|\sum_{l=1}^t\sum_{i=1}^n r_i (s_{iJ^l}(\theta)-\bar s_{iJ^l}^l)f_{i}(x^{(J^l)},\theta)  \right|> \max\limits_{i\in [n]}r_i \sqrt{32t\log(4|\Theta|t/\beta(n,T)) }  \right\}.
\end{equation*}
Then using \ref{lemma:azuma2} and the triangle inequality we obtain
\begin{equation*}
\mathbb{P}[A_n]\leq \beta(n,T)/t.
\end{equation*}
Thus by using the union bound we know the condition holds with probability at least $1-(\log T+1)\beta(n,T)$.
\end{proof}

We verify the second condition in step 3 of the algorithm is not violated for all $t\in [T]$ with high probability.

\begin{proposition}
\label{prop:iid second condition}
When $x^t$ arrives in a IID fashion, given $\mathcal{E}$, we have
\begin{equation*}
\mathbb{P}\left[ \forall~ t\in [T], \forall~ i\in [n]: \sum_{l=1}^t \bar s_{iJ^l}f_{i}(x^{(J^l)},\Omega^{l-1})\leq \frac{t}{T}c_i+\sqrt{2t\log(2t/\beta(n,T))} \right]\geq 1-(\log T +1)\beta(n,T).
\end{equation*}
\end{proposition}

\begin{proof}
Using \ref{lemma:azuma3}, it is easy to show
\begin{equation*}
\mathbb{P}\left[ \left|\sum_{l=1}^t \left(\bar s_{iJ^l}^lf_{i}(x^{(J^l)},\Omega^{l-1}) - \sum_{j=1}^L \mu_j \bar s_{ij}^lf_{i}(x^{(j)},\Omega^{l-1})    \right) \right|> \sqrt{2t\log(2t/\beta(n,T))} \right]\leq \beta(n,T)/t.
\end{equation*}
Since $\bar s_{ij}^l$ is the solution of $U^t$, we have for all $i\in [n]$ it holds almost surely
\begin{equation*}
\sum_{l=1}^t\sum_{j=1}^L \mu_j \bar s_{ij}^l f_{i}(x^{(j)},\Omega^{l-1})\leq \frac{t}{T}c_i.
\end{equation*}
Therefore the desired result holds by a union bound argument.
\end{proof}

Finally, we obtain $\tilde{O}(|\Theta|nT)$ regret under IID arrivals, as a corollary of Proposition \ref{Aprop:not switch sublinear regret}, Proposition \ref{prop:iid first condition} and Proposition \ref{prop:iid second condition}
.
\begin{corollary}
\label{Aprop:iid regret}
The expected regret under IID arrivals of the algorithm is upper bounded by
\begin{equation*}
16\sqrt{2|\Theta|nT\log(4|\Theta|nT^2)}+\sqrt{5}n\log(2nT^2)+(LP(\theta^*)/(nT)+1)(2/T+2+3\log T).
\end{equation*}
\end{corollary}

\subsubsection{Regret analysis: unified regret bound where $\text{ALG}_{\text{LP}}$ is switched to $\text{ALG}_{\text{ADV}}$ sometime}
\label{Asec:ULwE regret ADV}
In this section, we further analyze our algorithm to address the case where the switch occurs, transitioning from $\text{ALG}_{\text{LP}}$ to $\text{ALG}_{\text{ADV}}$. We extend our analysis to nonstationary arrivals and demonstrate that our algorithm achieves sublinear regret and a constant competitive ratio. By considering the combined performance of both algorithms, we derive a unified regret bound in Proposition \ref{Aprop:hybrid regret} that provides an overall measure of our algorithm's performance. This analysis allows us to quantify the regret incurred and evaluate the effectiveness of our approach in dynamically allocating resources.
\begin{proposition}
\label{Aprop:hybrid regret}
In any case of nonstationary arrivals, the algorithm guarantees 
\begin{equation*}
\text{OPT}\leq \left(1+ \frac{(1+\min\limits_{i\in [n]} c_i) \left(1-e^{-1/\min\limits_{i\in [n]}c_i }\right) }{1-1/e}\right)  \mathbb{E}[\text{ALG}]+\tilde{O}(\sqrt{n|\Theta|T} ).
\end{equation*}
\end{proposition}
\begin{proof}
Suppose the switch occurs in time epoch $t$. Denote $\text{OPT}_1$ the expected revenue obtained by the optimal algorithm (who knows $\theta^*$ at the beginning) from time epoch $1$ to $t$ and $\textrm{OPT}_2$ the rest of the expected revenue generated by the same algorithm. We use $\text{ALG}_{\text{LP}}^t$ and $\text{ALG}_{\text{ADV}}^{T-t+1}$ to denote the revenue obtained by our algorithm in the two phases, respectively. From Proposition \ref{Aprop:not switch regret} we know
\begin{equation*}
\text{OPT}_1\leq \mathbb{E}[\text{ALG}_{\text{LP}}^t] + 16\sqrt{2|\Theta|nt\log(\frac{4t}{\beta(n,T)})}+\sqrt{5}n\log(\frac{2t}{\beta(n,T)})+(\frac{2}{t}+2+\log t)\beta(n,T)(LP(\theta^*)+nt).
\end{equation*}
We define $\text{Empty-OPT}_2$ as the revenue obtained by the optimal algorithm from time $t+1$ to $T$, given the consumption of all resources is zero. We also define $\text{Empty-ALG}_{\text{ADV}}^{T-t+1}$   the revenue obtained by our algorithm (after the switch) in the second phase, given the consumption of all resources is zero. So from the definition, it holds almost surely
\begin{align*}
& \text{Empty-ALG}_{\text{ADV}}^{T-t+1}\leq \text{ALG}_{\text{ADV}}^{T-t+1}+\text{ALG}_{\text{LP}}^t.
\end{align*}
Thus, from Theorem \ref{Athm:adversarial regret} we know
\begin{align*}
\text{Empty-OPT}_2\leq & \frac{(1+\min\limits_{i\in [n]} c_i) \left(1-e^{-1/\min\limits_{i\in [n]}c_i }\right) }{1-1/e} \mathbb{E}[\text{Empty-ALG}_{\text{ADV}}^{T-t+1}]\\
& +\max\limits_{i\in [n]}r_i(\sqrt{n|\Theta|}+1)\sqrt{2(T-t+1) \log(2(T-t+1)/\beta(n,T)) }\\
& + \max\limits_{i\in [n]} r_i ((1/(T-t+1)+\log (T-t+1)+1)/n.
\end{align*}
For an arbitrary $0<\alpha<1$, we have
\begin{align*}
\text{ALG}_{\text{LP}}^t+\text{ALG}_{\text{ADV}}^{T-t+1} & = \alpha\text{ALG}_{\text{LP}}^t+(1-\alpha)\text{ALG}_{\text{LP}}^t+\text{ALG}_{\text{ADV}}^{T-t+1}\\
&\geq \alpha\text{ALG}_{\text{LP}}^t+(1-\alpha)(\text{ALG}_{\text{LP}}^t+\text{ALG}_{\text{ADV}}^{T-t+1}).
\end{align*}
Then plugging in the previous two inequalities on $\text{OPT}_1$, $\text{Empty-OPT}_2$ on the right hand side of the previous inequality, using $\text{OPT}=\text{OPT}_1+\text{OPT}_2$,  $\text{Empty-OPT}_2\geq\text{OPT}_2 $,  then we have
\begin{align*}
\text{OPT} &=\text{OPT}_1+\text{OPT}_2 \\
& \leq \mathbb{E}[\text{ALG}_{\text{LP}}^t] + \text{Empty-OPT}_2+ \tilde{O}(\sqrt{n|\Theta|T}) \\
& \leq  \mathbb{E}[\text{ALG}_{\text{LP}}^t] + \frac{(1+\min\limits_{i\in [n]} c_i) \left(1-e^{-1/\min\limits_{i\in [n]}c_i }\right) }{1-1/e} \mathbb{E}[\text{Empty-ALG}]+ \tilde{O}(\sqrt{n|\Theta|T})\\
& \leq  \mathbb{E}[\text{ALG}_{\text{LP}}^t] + \frac{(1+\min\limits_{i\in [n]} c_i) \left(1-e^{-1/\min\limits_{i\in [n]}c_i }\right) }{1-1/e} \mathbb{E}[\text{ALG}_{\text{ADV}}^{T-t+1}+\text{ALG}_{\text{LP}}^t]+ \tilde{O}(\sqrt{n|\Theta|T})\\
& \leq \left(1+ \frac{(1+\min\limits_{i\in [n]} c_i) \left(1-e^{-1/\min\limits_{i\in [n]}c_i }\right) }{1-1/e}\right)  \mathbb{E}[\text{ALG}]+\tilde{O}(\sqrt{n|\Theta|T} )
\end{align*} 

completes the proof.
\end{proof}
Through our analysis in Proposition \ref{Aprop:hybrid regret}, we demonstrate that the ULwE algorithm designed for nonstationary arrivals achieves sublinear regret, indicating diminishing regret growth as the time horizon increases. This sublinear growth rate is a desirable property, as it reflects the algorithm's ability to effectively adapt to the changing arrival patterns and make informed resource allocation decisions. We establish that our algorithm achieves a constant competitive ratio, highlighting its performance relative to an optimal decision-making strategy. This constant competitive ratio indicates that our algorithm achieves regret which is at most a constant factor compared to the optimal strategy, demonstrating its competitiveness in resource allocation even in the face of changing customer preferences.

\section{Appendix: Additional Computational Results}
\label{Asec:more results}
\subsection{Regret table under IID Arrivals}
\begin{table}[t]
\caption{Algorithm Solution Trajectories and Regret under IID Arrivals}
\label{tab:IID}
\vskip 0.15in
\begin{center}
\begin{small}
\begin{sc}
\begin{tabular}{lccccc}
\toprule
 & \multicolumn{2}{c}{Customer A} & \multicolumn{2}{c}{Customer B} &  \\
\cmidrule(r){2-3} \cmidrule(r){4-5}
Period & Resource 1 & Resource 2 & Resource 1 & Resource 2 & Regret \\
\midrule 
\multicolumn{6}{l}{LP Algorithm \& ULwE Algorithm} \\
\addlinespace
100 & 0 & 66 & 12 & 22 & 4.6  \\
200 & 0 & 126 & 38 & 36 & 12.4  \\
300 & 0  & 174 & 71 & 55 & 20.3\\
400 & 0 & 235 & 100 & 65 & 28.5\\
500 & 0 & 291 & 132 & 77 & 36.6\\
\addlinespace
\midrule
\multicolumn{6}{l}{ADV Algorithm} \\
\addlinespace
100 & 0 & 66 & 0 & 34 & 0 \\
200 & 19 & 107 & 10 & 64 & 10.7  \\
300 & 46 & 128 & 35 & 91 & 30.0\\
400 & 76 & 159 & 53 & 112 & 47.7\\
500 & 99 & 192 & 76 & 133 & 64.7\\
\bottomrule
\end{tabular}
\end{sc}
\end{small}
\end{center}
\vskip -0.1in
\end{table}
\cref{tab:IID} presents a detailed comparative analysis of the solution trajectories and regret for each algorithm under IID arrivals. It details resource allocation to Customer Types A and B and the corresponding regret values, demonstrating that $\text{ALG}_{\text{LP}}$ and ULwE outperform $\text{ALG}_{\text{ADV}}$. Notably, the $\text{ALG}_{\text{LP}}$ and ULwE Algorithms maintain consistent performance without violating inventory constraints, as evidenced by the zero values in the inventory regret column. This consistency corroborates the earlier observations regarding their effective resource allocation strategies.
\subsection{Regret Table under ADV Arrivals}
\begin{table}[t]
\caption{Algorithm Solution Trajectories and Regret under Adversarial Arrivals}
\label{tab:ADV}
\vskip 0.15in
\begin{center}
\begin{small}
\begin{sc}
\begin{tabular}{lccccr}
\toprule
& \multicolumn{2}{c}{Customer A} & \multicolumn{2}{c}{Customer B} & \\
\cmidrule(r){2-3} \cmidrule(r){4-5}
Period & Resource 1 & Resource 2 & Resource 1 & Resource 2 & Regret \\
\midrule
\multicolumn{6}{l}{LP Algorithm} \\
\addlinespace
100 & 0 & 13 & 77 & 10 & 20.5 \\
200 & 0 & 35 & 141 & 24 & 39.5 \\
300 & 0 & 75 & 193 & 32 & 55.8 \\
400 & 0 & 120 & 241 & 39 & 71.6 \\
500 & 0 & 153 & 250 & 77 & 88.1 \\
\midrule
\multicolumn{6}{l}{ADV Algorithm} \\
\addlinespace
100 & 0 & 13 & 0 & 87 & 0 \\
200 & 14 & 21 & 21 & 144 & 18.0 \\
300 & 44 & 31 & 66 & 159 & 40.1 \\
400 & 79 & 44 & 106 & 174 & 61.3 \\
500 & 100 & 53 & 150 & 197 & 81.8 \\
\midrule
\multicolumn{6}{l}{ULwE Algorithm} \\
\addlinespace
100 & 0 & 13 & 77 & 10 & 20.5 \\
200 & 0 & 35 & 121 & 44 & 37.1 \\
300 & 0 & 75 & 121 & 104 & 37.8 \\
400 & 27 & 93 & 155 & 125 & 55.7 \\
500 & 51 & 102 & 199 & 148 & 76.1 \\
\bottomrule
\end{tabular}
\end{sc}
\end{small}
\end{center}
\vskip -0.1in
\end{table}

\cref{tab:ADV} provides a detailed comparison of the algorithms' solution trajectories and regret. The ULwE Algorithm showed a distinct shift in its allocation patterns in the later stages, aligning more closely with $\text{ALG}_{\text{ADV}}$. Specifically, for the initial phase ($T < 200$), the ULwE Algorithm’s allocations closely resemble those of $\text{ALG}_{\text{LP}}$. In the middle phase ($200 < T < 300$), the ULwE Algorithm shifts its approach to mirror the early strategies of $\text{ALG}_{\text{ADV}}$, favoring Resource B predominantly. In the final phase ($T > 300$), the ULwE Algorithm adopts a more greedy allocation strategy, akin to the later strategies of $\text{ALG}_{\text{ADV}}$. This progression, also supported by \cref{fig:ADV1}, emphasizes the ULwE Algorithm’s capacity to dynamically adjust its allocation strategies in response to evolving conditions in nonstationary environments.

\subsection{Regret table under General Arrivals}
In this section, our objective is to assess the performance of $\text{ALG}_{\text{LP}}$, $\text{ALG}_{\text{ADV}}$, and the ULwE Algorithm under general customer arrivals. To investigate the impact of general customer arrivals, we consider specific arrival rate settings.  \cref{tab:arrival-rates} provides a description of the total arrival rates ($\lambda$), which represent the expected number of arrivals, for different settings. Each row in the table corresponds to a specific setting and includes information about the time period and the total arrival rates for customer types A and B. In the IID setting, the total arrival rate for customer type A is $0.6T$, and for customer type B, it is $0.4T$, over the entire time period $T$. In the ADV1 setting, the time period is divided into two parts: $0.33T$ and $0.67T$. In the first part, the arrival rates for customer types A is $0.8T$ and B is $0.2T$. In the second part, the arrival rates are reversed. Similarly, other settings, such as ADV2, follow a similar pattern where the arrival rates for customers change periodically. 

\begin{table}[!ht]
\centering
\caption{Description of different settings for customer arrival rates.}
\label{tab:arrival-rates}
\vskip 0.15in
\begin{small}
\begin{sc}
\begin{tabular}{@{}l p{2cm} p{2cm} p{2cm}@{}}
\toprule
Setting & Time Period $t$ ($\times T$) & Arrival Rate $\lambda_A$  ($\times T$) & Arrival Rate $\lambda_B$  ($\times T$) \\
\midrule
IID & 1 & 0.6 & 0.4 \\
\midrule
ADV1 & 0.33 & 0.15 & 0.85 \\
     & 0.67 & 0.4 & 0.6 \\
     \midrule
ADV2 & 0.1 & 0.2 & 0.8 \\
     & 0.3 & 0.8 & 0.2 \\
     & 0.2 & 0.2 & 0.8 \\
     & 0.1 & 0.4 & 0.6 \\
     & 0.1 & 0.2 & 0.8 \\
     & 0.1 & 0.02 & 0.98 \\
     & 0.1 & 0.2 & 0.8 \\
\bottomrule
\end{tabular}
\end{sc}
\end{small}
\vskip -0.1in
\end{table}

Regarding the performance of the IID and ADV1 settings, we have provided detailed explanations in the previous two subsections. In this section, our primary focus is directed toward the ADV2 settings, which introduce a heightened level of nonstationarity. We observe that $\text{ALG}_{\text{S}}$ consistently performs well throughout the entire period. Conversely, $\text{ALG}_{\text{ADV}}$ initially exhibits strong performance, followed by a period where its performance is similar to $\text{ALG}_{\text{LP}}$, and eventually surpasses $\text{ALG}_{\text{LP}}$ in the later stages. This implies that under conditions of high nonstationarity, $\text{ALG}_{\text{LP}}$ demonstrates the poorest performance.

Indeed, the findings indicate that the ULwE algorithm consistently performs well across different settings, highlighting its robustness in handling various scenarios. Specifically, it performs best in cases where there is a lower level of nonstationary arrival sequence, such as in the ADV1 setting. Since this algorithm's ability to switch from $\text{ALG}_{\text{LP}}$ to $\text{ALG}_{\text{ADV}}$ allows it to leverage the advantages of each approach, leading to superior performance. However, in scenarios with higher levels of nonstationarity, such as in the ADV2 settings, the performance of the ULwE algorithm may decrease but still be better than that in stationary settings. 

On the other hand, the performance of the ADV Algorithm is steady under nonstationary and the LP Algorithm demonstrates good performance under stationary conditions, but it struggles when faced with changes in arrival rates and higher levels of nonstationarity. This implies that the algorithm's static resource allocation strategy may not be suitable for dynamically changing environments. The effectiveness of the ULwE algorithm in handling nonstationary arrival patterns and adapting its resource allocation strategies is evident. By dynamically switching between different algorithms, it can effectively respond to changing conditions and achieve better performance while minimizing regret. This adaptability is a valuable characteristic that allows the ULwE algorithm to optimize resource allocation and cope with the uncertainties introduced by nonstationary arrival patterns.

\end{document}